\documentclass{article} % For LaTeX2e
\usepackage[preprint]{colm2025_conference}

\usepackage{microtype}
\usepackage{hyperref}
\usepackage{url}
\usepackage{booktabs}

\usepackage{lineno}

\usepackage[utf8]{inputenc} % allow utf-8 input
\usepackage[T1]{fontenc}    % use 8-bit T1 fonts
\usepackage{amsfonts}       % blackboard math symbols
\usepackage{nicefrac}       % compact symbols for 1/2, etc.
\usepackage{xcolor}         % colors
\usepackage{color}
\usepackage{listings}
\usepackage{caption}
\usepackage{pythonhighlight}
\usepackage{float}  % To create custom floating environments
\usepackage{graphicx}
\usepackage{placeins}
\usepackage{longtable}
\usepackage{natbib}
\setcitestyle{numbers}
\usepackage{subfig}
\usepackage{subcaption}
\usepackage{placeins}

\usepackage{dsfont}

\usepackage{amsmath}
\usepackage{amssymb}
\usepackage{amsthm}
\usepackage{amsfonts}

\usepackage{nicefrac}       % compact symbols for 1/2, etc.
\usepackage{microtype}      % microtypography
\usepackage{xcolor}         % colors
\usepackage{graphicx}
\usepackage{cleveref}
\usepackage{multirow}
\usepackage{mathtools}
\usepackage{thmtools}

\newtheorem{definition}{Definition}
\newtheorem{assumption}{Assumption}
\newtheorem{theorem}{Theorem}

\newtheorem{lemma}[theorem]{Lemma}

\newtheorem{remark}[theorem]{Remark}

\newcommand{\Mat}{\boldsymbol}
\newcommand{\Set}{\mathcal}

\newcommand{\real}{\mathbb{R}}

\newcommand{\Prob}{\mathbb{P}}
\newcommand{\mean}{\mathbb{E}}
\newcommand{\radem}{\mathcal{R}}
\newcommand{\LS}{L_{\Set{S}}}
\newcommand{\LD}{L_{\Set{D}}}
\newcommand{\LSp}{L_{\Set{S^\prime}}}
\newcommand\ie{\textit{i.e.}}

\DeclareMathOperator*{\argmax}{arg\,max}

\newfloat{pythoncode}{H}{mypython}
\floatname{pythoncode}{Code}

\definecolor{codegreen}{rgb}{0,0.6,0}
\definecolor{codegray}{rgb}{0.5,0.5,0.5}
\definecolor{codepurple}{rgb}{0.58,0,0.82}
\definecolor{backcolour}{rgb}{0.95,0.95,0.92}
\definecolor{lightgray}{rgb}{0.95,0.95,0.95}
\lstdefinestyle{pythonstyle}{
    language=Python,
    basicstyle=\ttfamily\small,
    keywordstyle=\color{blue},
    commentstyle=\color{green},
    stringstyle=\color{red},
    numbers=left,
    numberstyle=\tiny\color{gray},
    stepnumber=1,
    numbersep=10pt,
    showstringspaces=false,
    breaklines=true,
    frame=single,
    captionpos=b
}

\definecolor{darkblue}{rgb}{0, 0, 0.5}
\hypersetup{colorlinks=true, citecolor=darkblue, linkcolor=darkblue, urlcolor=darkblue}

\title{Sparse Mixture-of-Experts for Compositional Generalization: Empirical Evidence and Theoretical Foundations of Optimal Sparsity}

\author{Jinze Zhao$^{1}$, Peihao Wang$^1$, Junjie Yang$^2$, Ruisi Cai$^1$, Gaowen Liu$^3$,\\ \textbf{Jayanth Srinivasa$^3$, Ramana Rao Kompella$^3$, Yingbin Liang$^2$, Zhangyang Wang$^1$} \\
  $^1$University of Texas at Austin \\
  $^2$Ohio State University \\
  $^3$Cisco Research\\
  \texttt{\{jz24694, peihaowang\}@utexas.edu}}

\begin{document}

\ifcolmsubmission
\linenumbers
\fi

\maketitle

\begin{abstract}
Sparse Mixture-of-Experts (SMoE) architectures have gained prominence for their ability to scale neural networks, particularly transformers, without a proportional increase in computational cost. Despite their success, their role in compositional generalization, i.e., adapting to novel combinations of known components, remains under-explored. This study challenges the assumption that minimal expert activation suffices for task generalization and investigates the relationship between task complexity and optimal sparsity in SMoE models. Through empirical evaluations on the SRAVEN symbolic reasoning task and the SKILL-MIX benchmark, we demonstrate that (i) the number of activated experts consistently increases with the perceived task difficulty to maintain performance; and (ii) the optimal number of activated experts scales proportionally with task complexity. Our theoretical analysis derives a scaling law for optimal sparsity by balancing approximation and estimation errors, revealing alignment with empirical observations. We formally show that the optimal sparsity lies between minimal activation (1-2 experts) and full activation, with the exact number scaling proportionally to task complexity and further influenced by the size of the training data and the complexity of the model. These findings offer practical insights for designing SMoE models that achieve computational efficiency while enabling robust compositional generalization.
\end{abstract}

\vspace{-1em}
\section{Introduction}\label{sec:introduction}
\vspace{-1em}
The Sparse Mixture of Experts (SMoE) model, introduced to modern deep learning first by~\citet{shazeer2017outrageously}, has emerged as a compelling approach for scaling neural networks, particularly transformers, by significantly increasing model size without proportionally increasing computational costs. SMoE achieves this by partitioning the traditional feed-forward network into multiple homogeneous expert networks, dynamically and sparsely activated by a router. This modular architecture not only optimizes computational efficiency but also enhances generalization capabilities, especially in diverse data domains~\citep{mittal2022modular}. SMoE has become a standard architecture for many Large Language Models (LLMs) due to its superior performance and scalability~\citep{jiang2024mixtral,Grok-1,DBRX,dai2024deepseekmoeultimateexpertspecialization,qwen_moe,Samba-CoE-v0.3,lieber2024jambahybridtransformermambalanguage}.

Despite these advancements, the application of SMoE models to \textbf{compositional generalization} remains underexplored. Compositional generalization tasks require models to solve problems involving novel combinations of familiar components, with difficulty growing exponentially as the design space of possible combinations expands. For instance, in symbolic reasoning, a model may learn to solve individual arithmetic operations like addition and multiplication during training but must generalize to unseen combinations of these operations, such as nested expressions like \((3 + 5) \times 2\), which were not explicitly encountered before. The difficulty scales exponentially as the number of components and their interactions increase.
 Conventional SMoE configurations typically use minimal expert activation at some fixed, \textit{de-facto} sparsity (e.g., Top 1 or 2 out of 8 experts), a design that may become suboptimal for handling tasks of growing composition complexity. This raises a critical question: \textbf{Does the \textit{de-facto} activation sparsity  remain optimal as task complexity increases in compositional settings?}

To address this, we conduct both theoretical and empirical investigations, demonstrating that the optimal sparsity level for SMoE models scales proportionally with task complexity. Our contributions are summarized as follows:

\begin{itemize}
    \item We first trained SMoE-based transformers from scratch on the SRAVEN~\citep{schug2024attentionhypernetwork} synthetic symbolic reasoning task, varying task difficulty and the number of activated experts. Our findings reveal that activating more experts improves both Out-of-Distribution (OOD) generalization and test accuracy on harder tasks, with the optimal number of activated experts scaling roughly with the task's feature complexity.
    
    \item We then evaluated two pretrained SMoE-based LLMs, Mixtral-8$\times$7B~\citep{jiang2024mixtral} and DBRX-132B Instruct~\citep{DBRX}, on the SKILL-MIX benchmark~\citep{yu2023skill}, which challenges models to generate coherent text that integrates $k$ linguistic skills. Results show that activating more experts-per-token notably improves performance on harder tasks without additional training.

    \item Built on consistent empirical observations, we derived a theoretical scaling law for optimal sparsity in SMoE models under compositional input data settings, under simplified assumptions. Our analysis reveals that the optimal sparsity lies between minimal activation (e.g., 1-2 experts) and full activation of all experts, with the precise number depending on task difficulty, the size of the training data, and the complexity of the model.
\end{itemize}
Our work challenges the prevailing assumption that minimal expert activation is sufficient for complex tasks. By identifying the relationship between task difficulty and optimal sparsity, we provide actionable insights into designing SMoE models that balance computational efficiency with robust compositional generalization.
\vspace{-1em}
\section{Related Works}\label{sec:supp_relatedwork}
\vspace{-1em}
\paragraph{Compositional Generalization} 
Compositional generalization refers to a system's ability to understand and generate novel combinations of a finite set of familiar elements~\citep{fodor1988connectionism,chomsky2014aspects}. This capability is essential for enabling efficient learning and achieving robust generalization across diverse domains. In the computer vision field, studies have focused on generating images from new concept combinations, often using disentangled representation learning~\citep{esmaeili2018structureddisentangledrepresentations}. Researchers have evaluated VAE-based generative models in compositional tasks~\cite{zhao2018bias,montero2020role}, exploring the relationship between disentanglement and generalization performance in image reconstruction and generation. Recent studies have made significant strides by conducting a controlled investigation of compositional generalization in conditional diffusion models~\citep{okawa2024compositional,du2023reduce,leivada2023dall}, revealing insights into the emergence of compositional abilities and the factors influencing out-of-distribution generation. 

Recent works have observed emergent compositional capabilities in LLMs~\citep{wei2022emergent,yu2023skill,schaeffer2023emergentabilitieslargelanguage}. Several evaluation methods have been proposed to quantify compositional generalization of large, \textit{monolithic} pre-trained LLMs, including imposing generation constraints~\citep{yu2023skill}, multi-hop question answering~\citep{dziri2024faith}, and elementary math operations~\citep{lee2023teachingarithmeticsmalltransformers}. Theoretical advancements have also shed light on the conditions required for achieving compositional generalization in neural networks~\citep{schug2024attentionhypernetwork,schug2023discovering,arora2023theory}. More recently, \citet{huang2024hardertasksneedexperts} finds out that harder tasks need more experts and then proposed a heterogeneous routing strategy, and \citet{abnar2025parametersvsflopsscaling} investigates the sparsity scaling-law of SMoE, though both from non-compositional settings and without theoretical analysis. Despite substantial progress, there is a significant gap in understanding \textit{compositional generalization within modular architectures}, particularly SMoEs. Our work pioneers this investigation by exploring whether \textit{de-facto} sparse activation remains optimal as compositional task difficulty increases from both theoretical and experimental perspectives.

\paragraph{Theoretical Understanding of Mixture-of-Experts}
Recent work~\citep{chen2022towards} formally studied how the SMoE layer reduces training error more effectively than a single expert and why such a mixture model does not collapse into a single model. Importantly, when training an SMoE layer on data generated from a 'mixture of class' distribution using gradient descent, the authors proved that each expert in the SMoE model specializes in a specific portion of the data (at least one cluster), while the router learns the cluster-center features and routes inputs to the appropriate experts. Subsequently, a series of studies~\citep{nguyen2023convergence,nguyen2023demystifying,nguyen2023general,nguyen2023statistical,nguyen2024sigmoid} sought to establish the convergence rates of density estimation and parameter estimation in MoE models by defining Voronoi-based losses that describe the interaction between the gating function and experts, explaining why Top-1 gating enables faster convergence rates for parameter estimation compared to other gating mechanisms. More recently, \citet{jelassi2024mixture} theoretically justified that while the memorization performance consistently improves with an increasing number of experts in SMoE, reasoning capabilities tend to saturate. Meanwhile, \citet{akretche2024tighter} provided tighter risk bounds by incorporating local differential privacy into the gating mechanism to enhance the generalization ability of MoE. In contrast to those prior arts, we aim to formally study whether the conventional sparse activation of SMoE remains an optimal strategy for compositional tasks of increasing difficulty, using both theoretical and empirical approaches.

\vspace{-1em}
\section{Empirical Study}
\vspace{-1em}
We postpone the introduction of the used compositional tasks to~\Cref{sec:supp_more_related} due to page limit. We conduct two sets of experiments to investigate whether sparse expert activation remains optimal for handling compositional tasks of varying difficulty, in both training-from-scratch and testing pre-trained model settings. In training-from-scratch experiments, we trained and tested SMoE under varying experts sparsity levels and compositional task difficulty levels. We test pretrained SMoE-based LLMs with varying experts sparsity levels and compositional task difficulty levels since they are pre-trained under a fixed sparsity level, \textit{e.g.}, Top-2 routing.
Both set of experiments indicate that the \textit{de-facto} activate sparsity is non-optimal as the task complexity increases. 
%We postpone the introduction of these compositional tasks to~\Cref{sec:supp_more_related} due to page limit.
% \textcolor{red}{Add!}
\vspace{-1em}
\subsection{Training and Evaluting SMoE-based Transformers on SRAVEN}
\vspace{-1em}
We trained standard decoder-only SMoE-based transformers on SRAVEN synthetic tasks~\citep{schug2024attentionhypernetwork}. Each transformer block consists of multi-head attention %(e.g. softmax attention or hypernetwork linear attention proposed by~\cite{schug2024attentionhypernetwork}) 
with relative positional encoding~\citep{raffel2023exploringlimitstransferlearning} and feedforward layer (FFN). The feedforward layer in each block is an SMoE structure with 8 parallel homogeneous experts, where each expert is a 2-layer multi-layer perceptron (MLP). The router is a simple 1-layer dense layer with top-K softmax gating mechanism. For the SRAVEN hyperparameters, we fixed the grid size of the problem to be $3\times3$, which means that we fix the number of in-context examples for the model. We have $R=8$ possible rules to sample from. Following~\citet{schug2024attentionhypernetwork}, we split all possible rule combinations at difficulty level $M$ into training and testing set, also hold out 25\% of all possible combinations as OOD evaluation set. We can adjust the difficulty of the task by sampling $M\in \{1,2,\cdots,R\}$ different rules to compose the task. 
% We keep $25\%$ of all SRAVEN examples as OOD samples to evaluate if the model can generalize compositionally to unseen cases. \textcolor{red}{Add!}
\vspace{-1em}
\subsubsection{Sparse activation levels are not consistently optimal across all difficulty settings}\label{sec:main}
\vspace{-1em}
We trained SMoE transformers with various Top-K routing mechanisms for the same number of iterations across different difficulty levels of the SRAVEN task, keeping all other model and training hyperparameters constant. Surprisingly, Top-1 routing consistently achieved the worst Out-of-Distribution (OOD) accuracy across all difficulty levels. While Top-2 routing performed comparably to other more expensive routing mechanisms on easier tasks, its performance also declined significantly as the compositional task complexity increased, as shown in~\Cref{fig:moe_softmaxatt_ood}. Similar trends were observed in Test Accuracy evaluations, detailed in~\Cref{fig:moe_softmaxatt_test} and~\Cref{sec:supp_results}. These findings demonstrate that commonly-used ``de-facto" sparse activation mechanisms are suboptimal for learning compositional tasks, challenging previous conclusions from studies such as~\citet{shazeer2017outrageously,jiang2024mixtral}.

\begin{figure}[h]
\vspace{-10mm}
    \centering
    \includegraphics[width=1\linewidth]{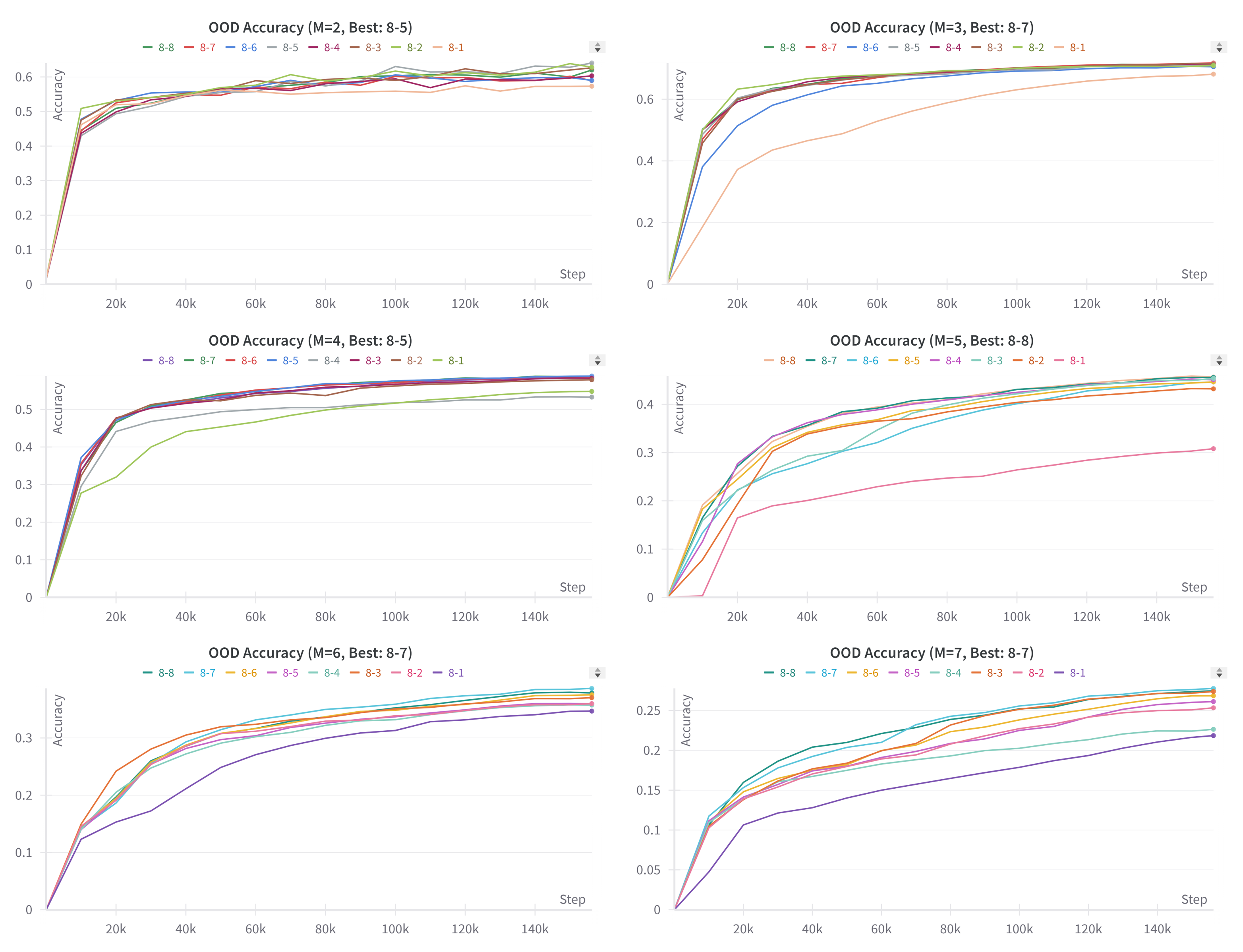}
    \caption{OOD accuracy of training SMoE Transformer, where the difficulty level of the task (i.e., the number of sampled rules $M$) increases. ''8-$k$'' means activating $k$ out of 8 experts on every FFN layer. The best-performing activation mechanism is labeled on the caption of each figure. }
    \label{fig:moe_softmaxatt_ood}
    \vspace{-5mm}
\end{figure}

\vspace{-1em}
\subsubsection{The optimal number of activated experts roughly scales with task difficulty}
\vspace{-1em}
We also observe that as we increase the compositional task difficulty, \ie, the number of sampled rules $M$, more activated experts are required to obtain the optimal performance correspondingly. Starting from the setting where $M=4$, the optimal number of activated experts $K$ is roughly scaled to $M$, and the performance gap between each activation mechanism also widens, as shown in~\Cref{fig:moe_softmaxatt_ood} and~\Cref{fig:moe_softmaxatt_test}. Therefore, we hypothesize that each expert can specialize in specific rules when the model is trained on compositional tasks, a phenomenon not typically observed in traditional NLP training tasks~\citep{fan2024towards}.

We also conducted ablation studies by switching from Softmax attention to HYLA attention~\citep{schug2024attentionhypernetwork}, a novel mechanism that encourages compositional generalization ability, and observed similar results as shown in~\Cref{sec:supp_ablation}. These observations collectively suggest that the choice of attention mechanism and the number of activated experts must be dynamically adjusted based on task difficulty. Specifically, HYLA attention coupled with increased expert activation is promising for robust compositional generalization in challenging tasks, while simpler tasks may not benefit significantly from these adjustments.

\vspace{-1em}
\subsection{Evaluating Pre-trained SMoE-based LLMs on Skill-Mix}
\vspace{-1em}
We evaluate two instruction-tuned SMoE-based LLMs, \ie, Mixtral-8$\times$7B Instruct-v0.1~\citep{jiang2024mixtral} and DBRX-132B Instruct~\citep{DBRX}, on the Skill-Mix~\citet{yu2023skill} benchmark. These evaluations utilize pre-trained models, meaning the training-time number of experts-per-token is already fixed and remains unchanged. Therefore, our experiments test the impact of varying the inference-time experts-per-token, essentially exploring the Out-of-Distribution (OOD) generalization of the number of active experts during the test phase. To ensure consistency in evaluation, we adopt GPT-4~\citep{openai2024gpt4technicalreport} as the grading model and use the released $10\%$ of the skill and topic lists from Section A of~\citet{yu2023skill}. The same optimized generation prompts are used for querying Mixtral and DBRX, and identical grading prompts and evaluation metrics are used to query GPT-4 for grading, as outlined in~\citet{yu2023skill}.

\vspace{-1em}
\subsubsection{More Experts-per-token Improves Compositional Generalization on Harder Tasks}
\vspace{-1em}
We evaluate Mixtral-8$\times$7B Instruct-v0.1 on Skill-Mix by varying both the task difficulty $k$ (i.e., the number of skills the model must combine in its generated response) and the number of experts-per-token (ept) during inference. Similar patterns are observed in the DBRX evaluation, as summarized in~\Cref{tab:skill-mix-evaluation} and~\Cref{tab:DBRX_skillmix}. Key findings include:

\begin{itemize}
    \item \textbf{Scaling Experts with Task Difficulty:} As the difficulty of the Skill-Mix task increases, activating more experts per token during inference is essential for performance. For example, the generated outputs from the default Top-2 routing setting score zero on all grading metrics when $k=4$, indicating that two experts are insufficient for such complex compositional tasks. In contrast, activating 4 or 5 experts achieves the best performance for $k=4$.
    
    \item \textbf{Optimal Experts Scale with Task Complexity:} The optimal number of experts per token scales with the task difficulty $k$. For simpler tasks (e.g., $k=1$ or $k=2$), a smaller number of experts (e.g., Top-2) suffices, while harder tasks (e.g., $k=4$) require more experts (e.g., 4 or 5 experts per token) to maintain compositional generalization.

    \item \textbf{Over-activation Can Harm Simpler Tasks:} Unlike the SRAVEN training experiments, which showed that increasing the number of activated experts does not degrade performance on simpler compositional tasks, the Skill-Mix evaluation of pre-trained LLMs highlights a downside of over-activation during inference. For instance, while the default Top-2 routing achieves perfect scores for $k=1$ and $k=2$ tasks, activating 7 or 8 experts during inference significantly lowers the model's performance on these simpler tasks, likely due to unnecessary complexity introduced by activating more experts than required.

\end{itemize}
\vspace{-1em}
These findings demonstrate that the optimal number of experts-per-token during inference is task-dependent, with harder tasks requiring more active experts for better compositional generalization. However, overactivating experts can negatively impact simpler tasks, highlighting the importance of dynamically adapting the number of experts based on task complexity. This insight emphasizes the need to treat training-time and inference-time sparsity settings separately to achieve robust performance across a range of task difficulties.

\begin{table}[htbp]
% \vspace{-5mm}
\centering
\resizebox{0.8\textwidth}{!}{%
\begin{tabular}{c|ccccc}
\toprule
Skill-Mix results for Mixtral-8$\times$7B evaluated by GPT4 & $k$ = 1 & $k$ = 2 & $k$ = 3 & $k$ = 4 & $k$ = 5 \\
\midrule
\multirow{3}{*}{EPT=1} & 0.00 $\pm$ 0.000 & 0.00 $\pm$ 0.000 & 0.00 $\pm$ 0.000 & 0.00 $\pm$ 0.000 & 0.00 $\pm$ 0.000 \\
 & 0.00 $\pm$ 0.000 & 0.00 $\pm$ 0.000 & 0.00 $\pm$ 0.000 & 0.00 $\pm$ 0.000 & 0.00 $\pm$ 0.000 \\
 & 0.00 $\pm$ 0.000 & 0.00 $\pm$ 0.000 & 0.20 $\pm$ 0.082 & 0.35 $\pm$ 0.100 & 0.00 $\pm$ 0.000 \\
\midrule
\multirow{3}{*}{EPT=2 (default)} & \textbf{1.00 $\pm$ 0.000} & \textbf{1.00 $\pm$ 0.000} & 0.00 $\pm$ 0.000 & 0.00 $\pm$ 0.000 & 0.00 $\pm$ 0.000 \\
 & \textbf{1.00 $\pm$ 0.000} & \textbf{1.00 $\pm$ 0.000} & \textbf{1.00 $\pm$ 0.000} & 0.00 $\pm$ 0.000 & 0.00 $\pm$ 0.000 \\
 & \textbf{1.00 $\pm$ 0.000} & \textbf{1.00 $\pm$ 0.000} & 0.00 $\pm$ 0.000 & 0.00 $\pm$ 0.000 & 0.00 $\pm$ 0.000 \\
\midrule
\multirow{3}{*}{EPT=3} & 0.00 $\pm$ 0.000 & 0.00 $\pm$ 0.000 & 0.00 $\pm$ 0.000 & 0.00 $\pm$ 0.000 & 0.00 $\pm$ 0.000 \\
 & 0.40 $\pm$ 0.245 & 0.20 $\pm$ 0.200 & 0.00 $\pm$ 0.000 & 0.00 $\pm$ 0.000 & 0.00 $\pm$ 0.000 \\
 & 0.00 $\pm$ 0.000 & 0.10 $\pm$ 0.100 & 0.47 $\pm$ 0.082 & 0.60 $\pm$ 0.061 & 0.00 $\pm$ 0.000 \\
\midrule
\multirow{3}{*}{EPT=4} & 0.20 $\pm$ 0.200 & 0.20 $\pm$ 0.200 & \textbf{0.20 $\pm$ 0.200} & \textbf{0.20 $\pm$ 0.200} & 0.00 $\pm$ 0.000 \\
 & 0.80 $\pm$ 0.200 & 0.40 $\pm$ 0.245 & \textbf{0.20 $\pm$ 0.200} & \textbf{0.20 $\pm$ 0.200} & \textbf{0.20 $\pm$ 0.200} \\
 & 0.20 $\pm$ 0.200 & 0.40 $\pm$ 0.187 & \textbf{0.67 $\pm$ 0.105} & \textbf{0.75 $\pm$ 0.079} & 0.00 $\pm$ 0.000 \\
\midrule
\multirow{3}{*}{EPT=5} & 0.20 $\pm$ 0.200 & 0.20 $\pm$ 0.200 & 0.20 $\pm$ 0.200 & \textbf{0.20 $\pm$ 0.200} & 0.00 $\pm$ 0.000 \\
 & 0.40 $\pm$ 0.245 & 0.20 $\pm$ 0.200 & 0.20 $\pm$ 0.200 & \textbf{0.20 $\pm$ 0.200} & \textbf{0.20 $\pm$ 0.200} \\
 & 0.20 $\pm$ 0.200 & 0.30 $\pm$ 0.200 & 0.53 $\pm$ 0.133 & \textbf{0.65 $\pm$ 0.100} & 0.00 $\pm$ 0.000 \\
\midrule
\multirow{3}{*}{EPT=6} & 0.00 $\pm$ 0.000 & 0.00 $\pm$ 0.000 & 0.00 $\pm$ 0.000 & 0.00 $\pm$ 0.000 & 0.00 $\pm$ 0.000 \\
 & 0.20 $\pm$ 0.200 & 0.00 $\pm$ 0.000 & 0.00 $\pm$ 0.000 & 0.00 $\pm$ 0.000 & 0.00 $\pm$ 0.000 \\
 & 0.00 $\pm$ 0.000 & 0.20 $\pm$ 0.122 & 0.47 $\pm$ 0.082 & 0.60 $\pm$ 0.061 & 0.00 $\pm$ 0.000 \\
\midrule
\multirow{3}{*}{EPT=7} & 0.00 $\pm$ 0.000 & 0.00 $\pm$ 0.000 & 0.00 $\pm$ 0.000 & 0.00 $\pm$ 0.000 & 0.00 $\pm$ 0.000 \\
 & 0.00 $\pm$ 0.000 & 0.00 $\pm$ 0.000 & 0.00 $\pm$ 0.000 & 0.00 $\pm$ 0.000 & 0.00 $\pm$ 0.000 \\
 & 0.00 $\pm$ 0.000 & 0.10 $\pm$ 0.100 & 0.40 $\pm$ 0.067 & 0.55 $\pm$ 0.050 & 0.00 $\pm$ 0.000 \\
\midrule
\multirow{3}{*}{EPT=8} & 0.00 $\pm$ 0.000 & 0.00 $\pm$ 0.000 & 0.00 $\pm$ 0.000 & 0.00 $\pm$ 0.000 & 0.00 $\pm$ 0.000 \\
 & 0.40 $\pm$ 0.245 & 0.20 $\pm$ 0.200 & 0.00 $\pm$ 0.000 & 0.00 $\pm$ 0.000 & 0.00 $\pm$ 0.000 \\
 & 0.00 $\pm$ 0.000 & 0.10 $\pm$ 0.100 & 0.47 $\pm$ 0.082 & 0.60 $\pm$ 0.061 & 0.00 $\pm$ 0.000 \\
\bottomrule
\end{tabular}%
}
\vspace{1em}
\caption{Skill-Mix Evaluation Results on Mixtral-8$\times$7B Instruct-v0.1~\cite{jiang2019fantastic}. The grading metrics are Ratio of Full Marks/Ratio of All Skills/Skill Fraction as defined in~\Cref{sec:supp_metrics}. 'EPT' stands for 'Number of experts per token.'}
\label{tab:skill-mix-evaluation}
\vspace{-4.5mm}
\end{table}
\vspace{-1em}

\section{Theoretical Analysis} \label{sec:theory}
\vspace{-1em}
In this section, we provide a novel theoretical analysis to justify our empirical findings.
We derive the generalization and approximation errors for SMoE trained on compositional tasks.
Based on this result, we show a theoretical trade-off between routing sparsity and task complexity, aligned with our empirical observations.
\vspace{-1em}
\subsection{Preliminaries} \label{sec:setting}
\vspace{-1em}
To begin with, we formally introduce the necessary mathematical setups for our successive analysis.
\vspace{-5mm}
\paragraph{Compositional Learning.}
Consider a data distribution: $\Set{D} \sim \Prob(\Mat{x}, y)$, where $\Mat{x} \in \real^d$ is an input feature vector and $y \in \Set{Y} \subseteq \real$ is the corresponding label. 
In compositional learning, we assume distribution $\Set{D}$ consists of $N$ tasks, and there are $N$ corresponding skills to solve each task, accordingly.
For each input $\Mat{x}$, we assume it can be represented by a subset of tasks and solved by corresponding skills, and its label $y$ is synthesized by combining the results of the corresponding skills.
Mathematically, we denote skills as a set of functions $\{t_1, \cdots, t_N\}$.
The data distribution is generated by $\Mat{x} \sim \Prob(\Mat{x})$, $y = G_{\Set{I}(\Mat{x})}(t_1(\Mat{x}), \cdots, t_N(\Mat{x}))$, where $\Set{I} \in 2^{[N]}$ indicates the task assignment of the input $\Mat{x}$, and $G_{\Set{I}}$ is the composition function according to $\Set{I}$.
We assume our training dataset $\Set{S}$ consists of $m$ i.i.d. samples drawn from the compositional data distribution: $\Set{S} = \{(\Mat{x}_1, y_1), \cdots, (\Mat{x}_m, y_m)\} \overset{i.i.d.}{\sim} \Set{D}^{m}$.
Given a parametric learner $f: \real^d \rightarrow \Set{Y}$, we optimize $f$ to minimize the empirical loss over the training set $\Set{S}$: $\LS(f) = \frac{1}{m} \sum_{i=1}^{m} \ell(f(\Mat{x}_i), y_i)$, where $\ell: \Set{Y} \times \Set{Y} \rightarrow \real$ is the element-wise loss function.
At the test stage, we consider the error within the whole data domain: $\LD(f) = \mean_{(\Mat{x}, y) \sim \Set{D}} \ell(f(\Mat{x}), y)$.
\vspace{-1em}
\paragraph{Sparse Mixture-of-Experts.}
We consider Sparse Mixture-of-Expert (SMoE) as the learner to the compositional task defined above.
SMoE can be defined as a data-dependent ensemble of many expert networks.
Suppose the total number of experts is $T$  and the number of dynamically activated experts is $k$.
Then SMoE can be written as a function $f: \real^d \rightarrow \Set{Y}$ defined as below:
\begin{align} \label{eqn:smoe}
f(\Mat{x}) = \sum_{j=1}^{T} a(\Mat{x})_j h_j(\Mat{x}) \quad \text{subject to } \sum_{j=1}^{T} \mathds{1}\{a(\Mat{x})_j \ne 0\} = k, \quad \forall \Mat{x} \in \real^{d},
\end{align}
where $a(\Mat{x}): \real^d \rightarrow \real^T$ is named as the routing function satisfying $\lVert a(\Mat{x}) \rVert_0 = k$, and $h_j(\Mat{x}): \real^d \rightarrow \Set{Y}$ is an expert network for every $j = 1, \cdots, T$ \footnote{For simplicity, We assume the (convex) linear combination of $\Set{Y}$ is still in $\Set{Y}$.}.
Intuitively, $a(\Mat{x})$ selects $k$ experts to be activated to inference labels for $\Mat{x}$.
In this paper, we consider the normalized gating function, which first chooses the $k$ logits from the output, sets the remainders to zeros, and applies softmax to normalize the chosen entries:
\begin{align} \label{eqn:gate_output}
a(\Mat{x})_j &= \left\{
\begin{array}{ll}
\frac{\exp(g(\Mat{x})_j)}{\sum_{t \in \Set{J}(\Mat{x})} \exp(g(\Mat{x})_t) } & \text{if } j \in \Set{J}(\Mat{x}) \\ \\
0 & \text{if } j \notin \Set{J}(\Mat{x})
\end{array}\right. ,
\end{align}
where $g: \real^d \rightarrow \real^T$ computes the weight for each expert, and $\Set{J}(\Mat{x})$ finds a sparse mask with at most $k$ non-zero entries according to $\Mat{x}$, i.e., $\lvert \Set{J}(\Mat{x}) \rvert = k, \forall \Mat{x} \in \real^{d}$.
Most typically, $\Set{J}(\Mat{x})$ selects the indices corresponding to the top-$k$ largest logits from $g(\Mat{x})$ \cite{shazeer2017outrageously}.

Now we can formally state the hypothesis space of an SMoE model, which is a composition of both the hypothesis spaces of gating and expert networks.
\begin{definition} \label{dfn:smoe}
Suppose all expert networks $h_1, \cdots, h_T \in \Set{H}$ is selected from the same hypothesis space $\Set{H}$, and $k$-sparse routing function $a \in \Set{A}$ is chosen from the hypothesis space $\Set{A}$. Define the hypothesis space of the SMoE model with $T$ experts and $k$-sparse routing function following Eq. \ref{eqn:smoe} and Eq. \ref{eqn:gate_output} as below:
\begin{align}
\Set{F}(T, k) = \left\{ f(\Mat{x}) = \sum_{j=1}^{T} a(\Mat{x})_j h_j(\Mat{x}) : h_1, \cdots, h_T \in \Set{H}, a \in \Set{A} \right\}
\end{align}
\end{definition}

\paragraph{Complexity Metrics.}
We will also employ two complexity metrics classical in learning theory to characterize the generalization error of SMoE.
First, we consider Rademacher complexity, which directly measures the capacity of a model class in terms of its ability to fit random labels, and it can depend on the specific data distribution.
Given a class of functions $\mathcal{H}$,  data distribution $\Set{D}$ and samples $S=\left\{z_1, \ldots, z_m\right\}$ drawn i.i.d. from $\Set{D}$, we can define Rademacher complexity as below:
\begin{definition}[Rademacher complexity] \label{dfn:radem_complexity}
The empirical Rademacher complexity of $\mathcal{H}$ is defined to be     
\begin{align}
\radem_m(\Set{H}, \Set{S})=\mean_\sigma\left[\sup _{f \in \mathcal{H}}\left(\frac{1}{m} \sum_{i=1}^m \sigma_i f\left(z_i\right)\right)\right]
\end{align}
where $\sigma_1, \ldots, \sigma_m$ are independent random variables uniformly chosen from $\{-1,1\}$. We will refer to such random variables as Rademacher variables.
The Rademacher Complexity of $\Set{H}$ is defined as $\radem_m(\mathcal{H}) = \mean_{\Set{S} \sim \Set{D}^m} \left[ \radem_m(\Set{H}, \Set{S}) \right]$.
\end{definition}

We also consider a combinatorial measure of model capacity -- \textit{Natarajan dimension}, which often provides tighter bounds for finite decision regions \citep{jin2023upper}. 
Natarajan dimension is a generalization of the VC dimension to classes of multiclass predictors \citep{shalev-shwartz_ben-david_2022}.
We formally state it below, which requires us to first generalize the definition of shattering.
\begin{definition}[Natarajan Dimension] \label{dfn:natarajan_dim}
We say that a set $\Set{C} \subset \mathcal{X}$ is shattered \cite{shalev-shwartz_ben-david_2022} by hypothesis space $\mathcal{H}$ if there exist two functions $f_0, f_1: \Set{C} \rightarrow[k]$ such that i) for every $\Mat{x} \in \Set{C}, f_0(\Mat{x}) \neq f_1(\Mat{x})$; ii) for every $\Set{B} \subset \Set{C}$, there exists a function $h \in \mathcal{H}$ such that $\forall \Mat{x} \in \Set{B}, h(\Mat{x})=f_0(\Mat{x})$ and $\forall \Mat{x} \in \Set{C} \backslash \Set{B}, h(\Mat{x})=f_1(\Mat{x})$.
The Natarajan dimension of $\Set{H}$, denoted by $\mathrm{NDim}(\Set{H})$, is the maximal size of a shattered set $\Set{C} \subset \Set{X}$.
\end{definition}
When $k = 2$, this definition degenerates to the VC dimension.
Interested readers are referred to \citet{shalev-shwartz_ben-david_2022} for more details.
\vspace{-1em}
\subsection{Generalization Error Analysis}
\vspace{-1em}
In this section, we analyze the generalization error defined as the discrepancy between the empirical and population losses: $\lvert \LS - \LD \rvert$.

First of all, we quantify the complexity of a family of routing functions.
Our approach is to disentangle the gating output by normalized weights multiplied with a mask: $a(\Mat{x}) = m(\Mat{x}) \odot \Mat{\nu}(\Mat{x})$, where $m(\Mat{x})_j = 1$ if $a(\Mat{x})_j \ne 0$, otherwise $m(\Mat{x})_j = 0$.
We note that $m(\Mat{x}): \real^d \rightarrow \{0, 1\}^T$ is a multi-class classifier.
Henceforth, we can characterize the complexity of $m(\Mat{x})$ via the Natarajan dimension.
Define a family of masking functions $\Set{M}$ induced by the class of gating functions $\Set{A}$.
Then the complexity of $\Set{M}$ is specified by the following assumptions.
\begin{assumption}
\label{ass:nat_dim}
The Natarajan dimension of $\Set{M}$ is scaled with the number of tasks $N$ as $\mathrm{NDim}(\Set{M}) = O(Nd_N)$ where $d_N$ is the base case when $N=1$.
\end{assumption}
$\mathrm{NDim}(\Set{M})$ represents the maximal cardinality of a set that shatters all possible outcomes of $m(\Mat{x})$, which can be used for counting the number of sparse patterns produced by $a(\Mat{x})$
In Assumption \ref{ass:nat_dim}, the capacity of the sparse router is growing with the number of tasks because more tasks often cause higher complexity of compositional data distribution, which yields more diverse expert combinations to solve the compositional tasks.

Next, we make basic assumptions on the loss function:
\begin{assumption}
\label{ass:loss}
The loss function $\ell: \Set{Y} \times \Set{Y} \rightarrow [0, 1]$ is $C$-Lipschitz.
\end{assumption}
Such an assumption is standard in generalization error analysis \citep{shalev-shwartz_ben-david_2022}. 
Essentially, we hypothesize that the loss function is bounded and Lipschitz continuous, which is satisfied by many common choices (e.g. cross-entropy or MSE) when the inputs or outputs are restricted to a compact space.

Now we can state the main result bounding the generalization error of SMoE under the compositional learning setting:
\begin{theorem} \label{thm:main}
Consider the hypothesis space $\Set{F}(T, k)$ stated in Definition \ref{dfn:smoe}, under Assumptions \ref{ass:nat_dim} and \ref{ass:loss}, with probability at least $1 - \delta$ over the selection of training samples, the generalization error is upper bounded by:
\begin{align}
\lvert \LS - \LD \rvert = O\left(4C\radem_m(\Set{H}) + 2\sqrt{\frac{2 k N d_N \log T + Nd_N \log(2m) + \log(2 / \delta)}{2m}}\right),
\end{align}
where $\radem_m(\Set{H})$ is the Rademacher complexity of the expert hypothesis space $\Set{H}$ (cf. Definition \ref{dfn:radem_complexity}).
% $d_N$ is the Natarajan dimension of gating function hypothesis space $\Set{A}$ (cf. Definition \ref{dfn:natarajan_dim_gating}), $m$ is the number of training samples, $T$ is the total number of experts, and $k$ is the number of selected experts.
\end{theorem}
The proof is elementary and deferred to Appendix \ref{sec:app:proofs}.
Theorem \ref{thm:main} reveals that SMoE can be helpful for compositional learning. This can be seen in the comparison with the naive ensemble, where $T$ weak models are stacked to form the stronger model.
Intuitively, the entire SMoE model contains $T$ copies of expert learners, thus, its capacity is as large as the naive ensemble.
However, the generalization error of the naive ensemble follows the classical results and gives a rate growing linearly with the number of total experts $O(T\radem(\Set{H}) + \sqrt{\log(2/\delta)/2m})$, while SMoE only exhibits a logarithmic dependency on $T$ and linear dependency on $k$.
This suggests that dynamic routing not only reduces the inference cost but also shrinks the average model complexity.
From the perspective of each data point, only a part of the model is used.

Nevertheless, dynamic routing incurs additional costs. Specifically, it introduces an extra term $N d_N$ in the error bound due to the flexibility provided by the learnable router.
Furthermore, the composition of multiple tasks exacerbates the bound, as the increased task complexity demands a greater variety of routing patterns, making SMoE harder to generalize.
\vspace{-1em}
\subsection{Optimal Sparsity Analysis}
\vspace{-1em}
% Non-Overlapping Task Assignments + limited Effective Capacity per expert,
In this section, we aim to analyze the optimal sparsity $k$ the model should choose to obtain the least error. We first decompose the generalization error obtain from~\Cref{thm:main} into a different form as the summation of \textbf{Approximation Error} and \textbf{Estimation Error} under moderate assumptions, where both error terms are derived independently based on sparsity and input task complexity, and then analyze the optimal choice of $k$ under different hyperparameter settings.
% \vspace{-1em}
\subsubsection{Approximation Error Construction.}
% \vspace{-1em}
In this section, we outline the assumptions required to characterize the Approximation Error first, and then construct the Approximation Error.
\begin{assumption}\label{ass:approx1}
    Each input $\Mat{x}$ is a composition of two tasks sampled from a total number of $N$ tasks, \ie, $\Mat{x} = T_i \circ T_j$ where $T_i, T_j \in \{T_1, T_2, \dots, T_N\}$.
\end{assumption}

\begin{assumption}\label{ass:approx2}
    Each expert $h_i$ is uniformly assigned with non-overlapping compositional tasks. 
\end{assumption}

\begin{assumption}\label{ass:approx3}
    Each single task $T_i$ is equally weighted and does not account for task-specific variability such as importance or difficulty of each task.
\end{assumption}
We adopt these simplified settings to help with the theoretical analysis. A more complex and realistic \textit{Power-Law Distribution}~\citep{zhong2024law} setting for compositional tasks is described in~\Cref{supp:powerlaw}. Under the current simplified setting, we define the Approximation Error of SMoE as following:
\begin{definition}[Approximation Error]\label{def:approx_error} 
    Under Assumptions \ref{ass:approx1} \ref{ass:approx2} \ref{ass:approx3}, the Approximation Error of SMoE is modeled as:
    \begin{align} \label{eqn:approx_error}
        E_{\text{approximation}}(k, N) = O\left(\frac{N^2}{k}\right),
    \end{align}
    where $N^2$ is the total number of pairwise task compositions, and $k$ is the total number of selected experts.
\end{definition}

\begin{remark}
    The Approximation Error $O\left(\frac{N^2}{k}\right)$ is independent of the number of input data $m$ and it treats each task uniformly and distributes each experts' capacity evenly across all $N^2$ possible combinations. This error term decreases monotonically as \(k\) increases, and when $k$ is small it dominates due to insufficient expert capacity for task combinations.
\end{remark}
% \vspace{-1em}
\subsubsection{Estimation Error Construction.}
% \vspace{-1em}
In this section, we construct the Estimation Error based on~\Cref{thm:main} to capture the variations in the generalization error given different sparsity levels, number of input data, and the complexity of the router.

\begin{definition}[Estimation Error]\label{def:esti_error}  
    Under Assumptions \ref{ass:approx1} \ref{ass:approx2} \ref{ass:approx3}, the Estimation Error of SMoE is modeled as:
    \begin{align*}~\label{eqn:esti_error}
        E_{\text{estimation}}(k, N) = O\left(\sqrt{\frac{N k d_N}{m} \log\left(\frac{T}{k}\right)}\right),
    \end{align*}
    where $N$, $k$, $d_N$, $m$, and $T$ follows their definitions in~\Cref{sec:setting}.
\end{definition}

\begin{remark}
    Regarding the Estimation Error, we observe that it is a concave function with respect to $k$, where the error increases and reaches the maximum at $k = \frac{T}{e}$, and then decreases as $k \to T$. The presence of \(N\) penalizes large \(k\) more strongly, making the estimation error scale with both the number of tasks \(N\) and the number of active experts \(k\).
\end{remark}

\subsubsection{Bias-Complexity Trade-off}
% \textcolor{red}{write about optimal k selection, and power law distribution of tasks}
In this section, we provide insights on how does the total error $E_{\text{total}}(k, N) = O\left(\frac{N^2}{k}\right) + O\left(\sqrt{\frac{N k d_N}{m} \log\left(T/k\right)}\right)$ scales with different parameter settings, helping us choose the optimal $k^{*}$ to minimize the overall error.
\vspace{-1em}
\paragraph{$k$ Dependence (Number of Selected Experts Per Task):}
The behavior of the error terms varies significantly with k. For small k, the approximation error $O\left(\frac{N^2}{k}\right)$ dominates while the estimation error remains minimal. As k reaches medium values $(k \approx \frac{T}{e})$, the estimation error peaks while the approximation error begins to decrease. For large k when $k \to T$, the approximation error approaches zero as expert capacity becomes sufficient, while the estimation error decreases but remains non-negligible due to $N$ scaling. The optimal $k^*$ must balance these competing errors.
\vspace{-1em}
\paragraph{$N$ Dependence (Task Complexity):}
Both approximation and estimation errors scale with $N$, but the approximation error grows faster at large $k$. For large \( N \), larger \( k \) values are favored to reduce the approximation error.
\vspace{-1em}
\paragraph{$d_N$ Dependence (Model Complexity):}
A high router complexity $d_N$ amplifies the estimation error penalty for large $k$, favoring $k^{*} < T$. 
\vspace{-1em}
\paragraph{$m$ Dependence (Dataset Size):}
The estimation error diminishes with increasing $m$, reducing its impact at large $k$. This allows $k = T$ to minimize $E(k, N)$. For large $m$, the approximation error tends to dominate. For small $k$, it may gradually turn favoring $k^{*} < T$.
\vspace{-1em}
\paragraph{$T$ Dependence (Total Experts):}
Increasing \( T \) increases the estimation error if \( k \) is kept constant. To prevent the estimation error from increasing, \( k \) should be increased proportionally with \( T \).

Overall, we conclude that (a): Under large dataset regime ($m$ is large), a larger $k^{*}$ is preferred. (b): Under high router complexity regime ($d_N$ is large), a smaller $k^{*}$ is preferred to mitigate the increased estimation error. (c): Under increasing task complexity setting, a larger $k^{*}$ is preferred. (d): When the total number of available experts $T$ is increasing, $k$ needs to be adjusted proportionally to prevent an increase in the estimation error.

\section{Conclusion}
In this work, we investigated whether conventional sparse activation strategies in SMoE models are optimal for compositional tasks. Through both empirical and theoretical analyses, we showed that the optimal number of activated experts scales with task complexity and depends on training data size and model architecture. Our experiments revealed that increasing the number of experts improves generalization on harder compositional tasks but can degrade performance on simpler ones, highlighting the need for adaptive sparsity. Theoretically, we derived generalization error bounds for SMoE models under compositional settings, identifying a trade-off between approximation and estimation errors. These findings challenge traditional sparse activation assumptions and provide actionable insights for designing SMoE models. Future work will explore dynamic routing strategies to adapt expert activation levels based on real-time task complexity.

\newpage
\bibliography{colm2025_conference}
\bibliographystyle{colm2025_conference}

\newpage
\appendix
\section{Appendix}
\subsection{More Related Works}\label{sec:supp_more_related}
\paragraph{SRAVEN symbolic reasoning test}
Inspired by Raven Progressive Matrices (RAVEN) Test~\citep{raven2008uses}, a human Intelligence Quotient test on abstract reasoning,~\citet{schug2024attentionhypernetwork} proposed the symbolic compositional SRAVEN test that requires a model to learn the composition of arbitrarily sampled rules by searching through a large number of possible hypotheses. Similar to RAVEN, each SRAVEN task is a $3\times3$ grid and the model is asked to query the final panel on the grid given the information from the first 8 panels. Each panel is a vector of length $M$, where each entry on the vector corresponds to a different rule sampled from $R=8$ possible rules. Therefore, each task is composed by a finite set of rule combinations, and the prediction will be marked as correct only if the model predict all the entries of the final vector. More detailed description can be found in Section 4 in~\citep{schug2024attentionhypernetwork}. In this paper, we trained SMoE-based transformers on SRAVEN task due to its compositional nature and the flexibility on tuning the difficulty of the task.

\paragraph{Skill-Mix}
Skill-Mix~\citep{yu2023skill} is a novel evaluation method for assessing language models' compositional abilities. It challenges models to generate short text pieces combining random sets of $k$ skills out of $N$ number of linguistic skills within a given topic. Therefore, the test's difficulty increases with $k$. A Grader model (e.g., GPT-4~\citep{openai2024gpt4technicalreport}) is used to evaluate the generated outputs based on skill application, topic relevance, length, and coherence. In this paper, we tested SMoE-based LLMs on Skill-Mix with varying number of $k$. More experimental and grading details can be found in Section C in~\citet{yu2023skill}.

\subsection{Training SMoE-based Transformers on SRAVEN task}

\subsubsection{Ablation study: Switching Softmax attention to HYLA attention}\label{sec:supp_ablation}
Proposed by~\citet{schug2024attentionhypernetwork}, Hypernetwork Linear Attention (HYLA) encourages compositional generalization by reinforcing the hypernetwork perspective. We repeat the same set of experiments by replacing softmax attention with HYLA attention. We have the following interesting findings:
\begin{itemize}
    \item By comparing~\Cref{fig:moe_hypatt_ood} and~\Cref{fig:moe_softmaxatt_ood}, we observe that HYLA does not improve and can even degrade compositional generalization (particularly OOD accuracy) on SMoE models when training on easier compositional tasks (i.e., $M=2$) compared to using softmax attention. However, as the compositional task becomes more challenging, HYLA significantly outperforms softmax attention, enhancing the compositional generalization of the model across all activation mechanisms.
    
    \item As shown in~\Cref{fig:moe_hypatt_test} and~\Cref{fig:moe_hypatt_ood}, sparse activation mechanisms like Top-1 and Top-2 remain the worst-performing routing strategies across almost all task difficulty levels, even with HYLA attention. This finding aligns with our main results discussed in~\Cref{sec:main}.
    
    \item As illustrated in~\Cref{fig:moe_softmaxatt_test} and~\Cref{fig:moe_softmaxatt_ood}, under the softmax attention setting, increasing the number of activated experts leads to marginal improvements in both Test and OOD accuracy as the compositional task difficulty increases. In contrast, under HYLA attention, the same increase in activated experts results in dramatic improvements for more challenging compositional tasks, as shown in the last row of~\Cref{fig:moe_hypatt_test} and~\Cref{fig:moe_hypatt_ood}. The optimal number of activated experts scales proportionally with task difficulty $M$, and the performance gap between the best and second-best routing mechanisms becomes significant. We hypothesize that HYLA attention further enhances expert specialization for compositional tasks.

    Overall, HYLA attention coupled with increased expert activation is promising for robust compositional generalization in challenging tasks, while simpler tasks may not benefit significantly from these adjustments.
\end{itemize}
\newpage
\subsubsection{SRAVEN Experiments Results Details}\label{sec:supp_results}
\begin{figure}[!htb]
    \centering
    \includegraphics[width=0.8\linewidth]{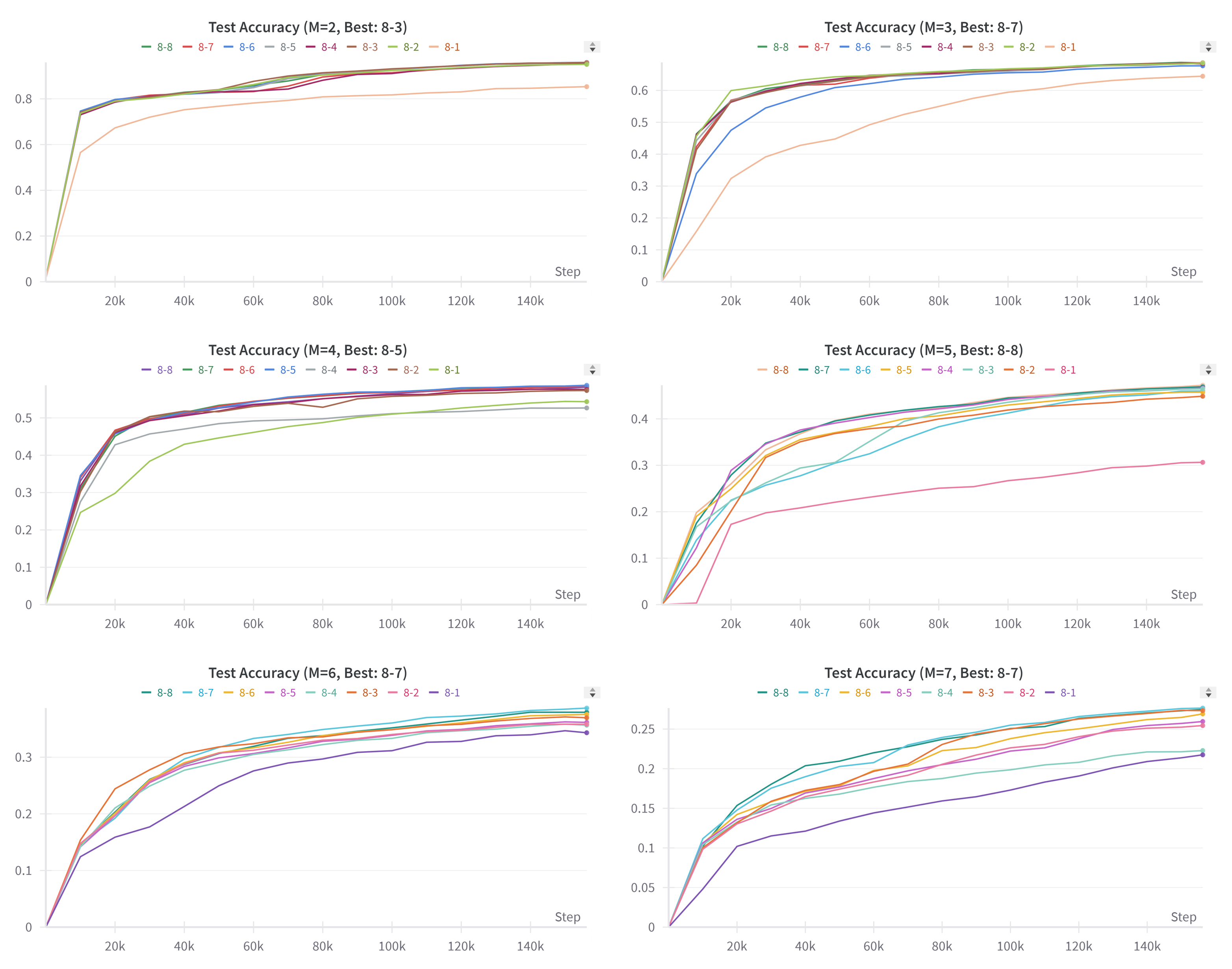}
    \caption{Test Accuracy of training SMoE Transformer with softmax attention.}
    \label{fig:moe_softmaxatt_test}
\end{figure}

\begin{figure}[!htb]
    \centering
    \includegraphics[width=0.8\linewidth]{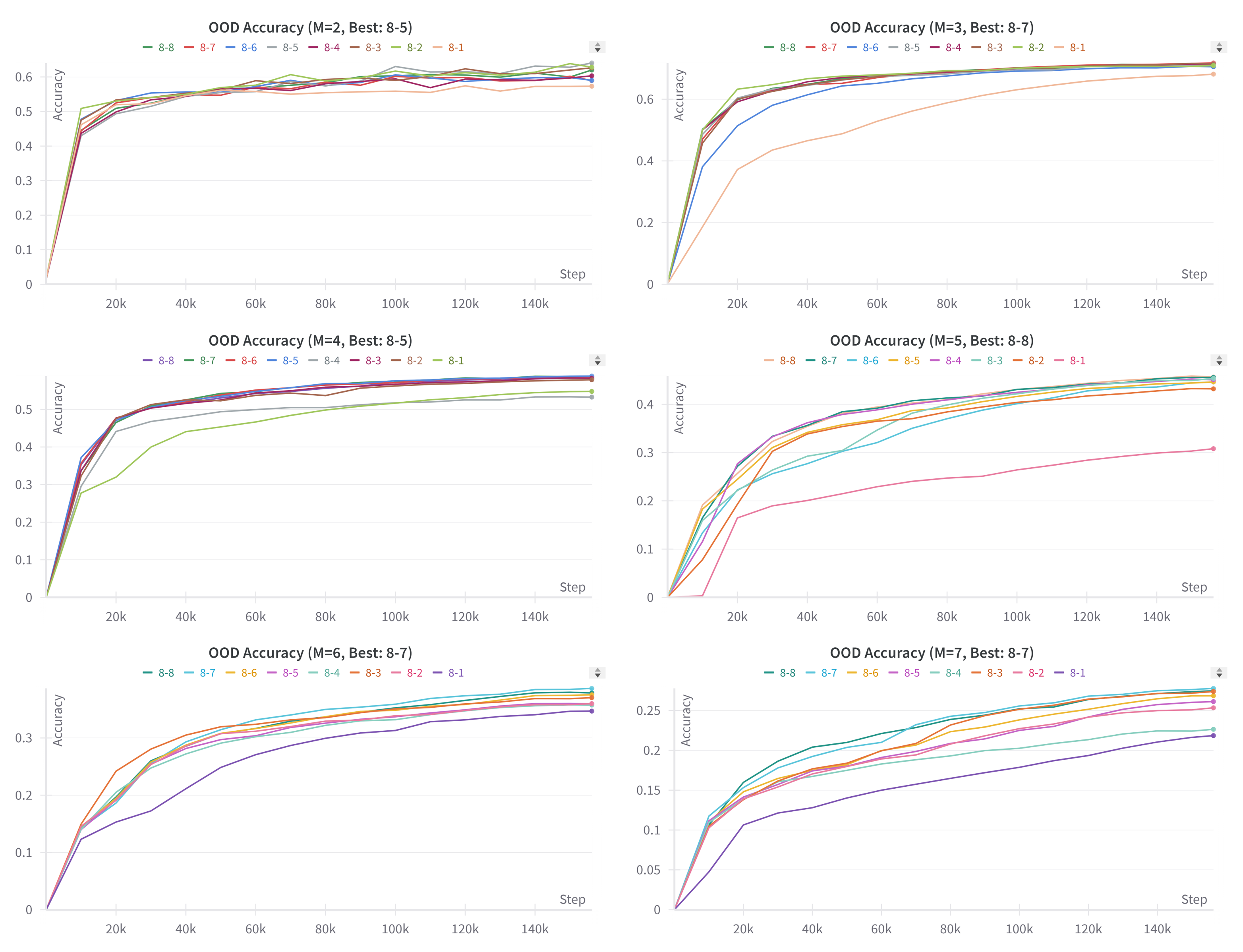}
    \caption{OOD Accuracy of training SMoE Transformer with softmax attention.}
\end{figure}

\begin{figure}[!htb]
    \centering
    \includegraphics[width=0.8\linewidth]{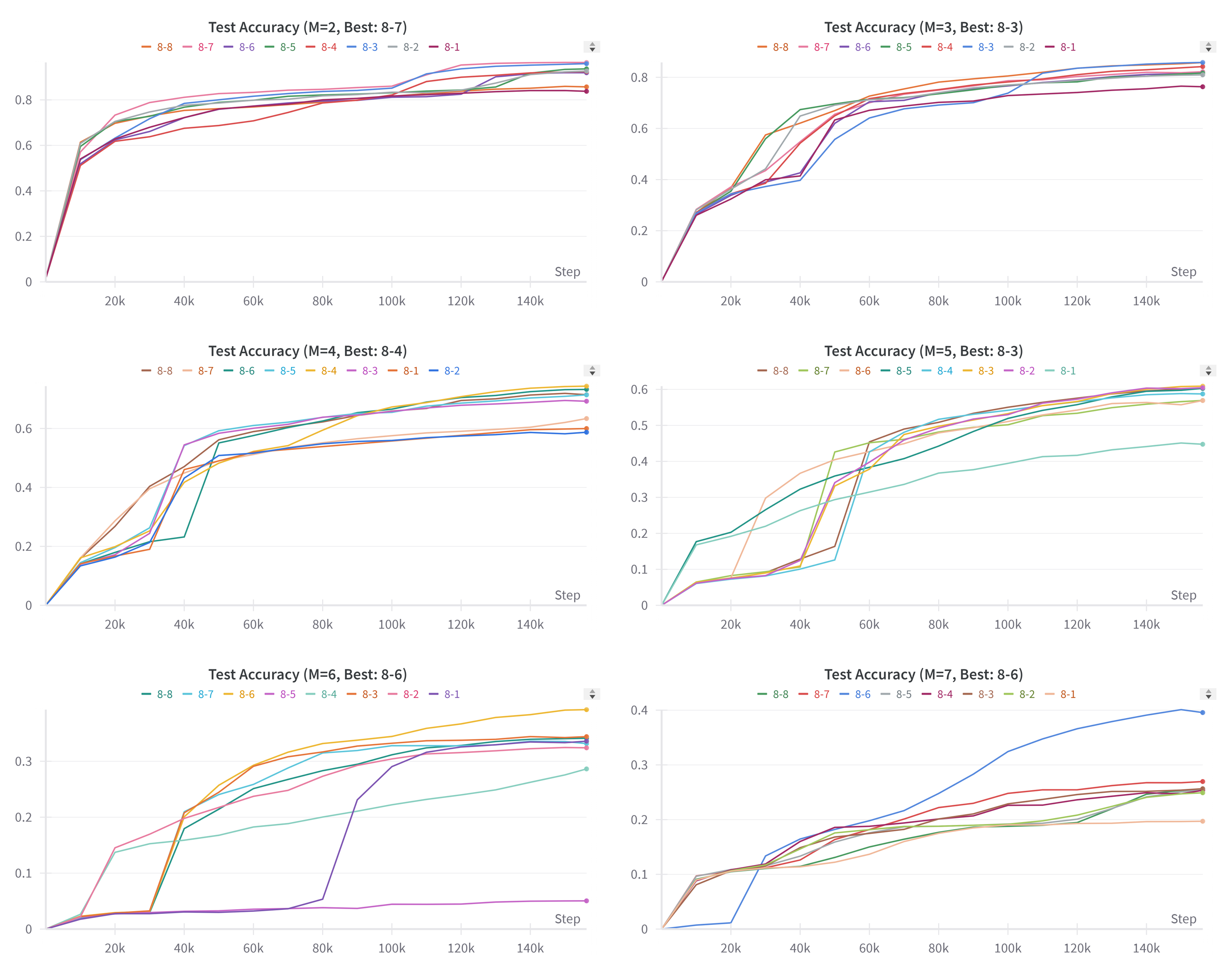}
    \caption{Test Accuracy of training SMoE Transformer with hypernetwork linear attention.}
    \label{fig:moe_hypatt_test}
\end{figure}

\begin{figure}[!htb]
    \centering
    \includegraphics[width=0.8\linewidth]{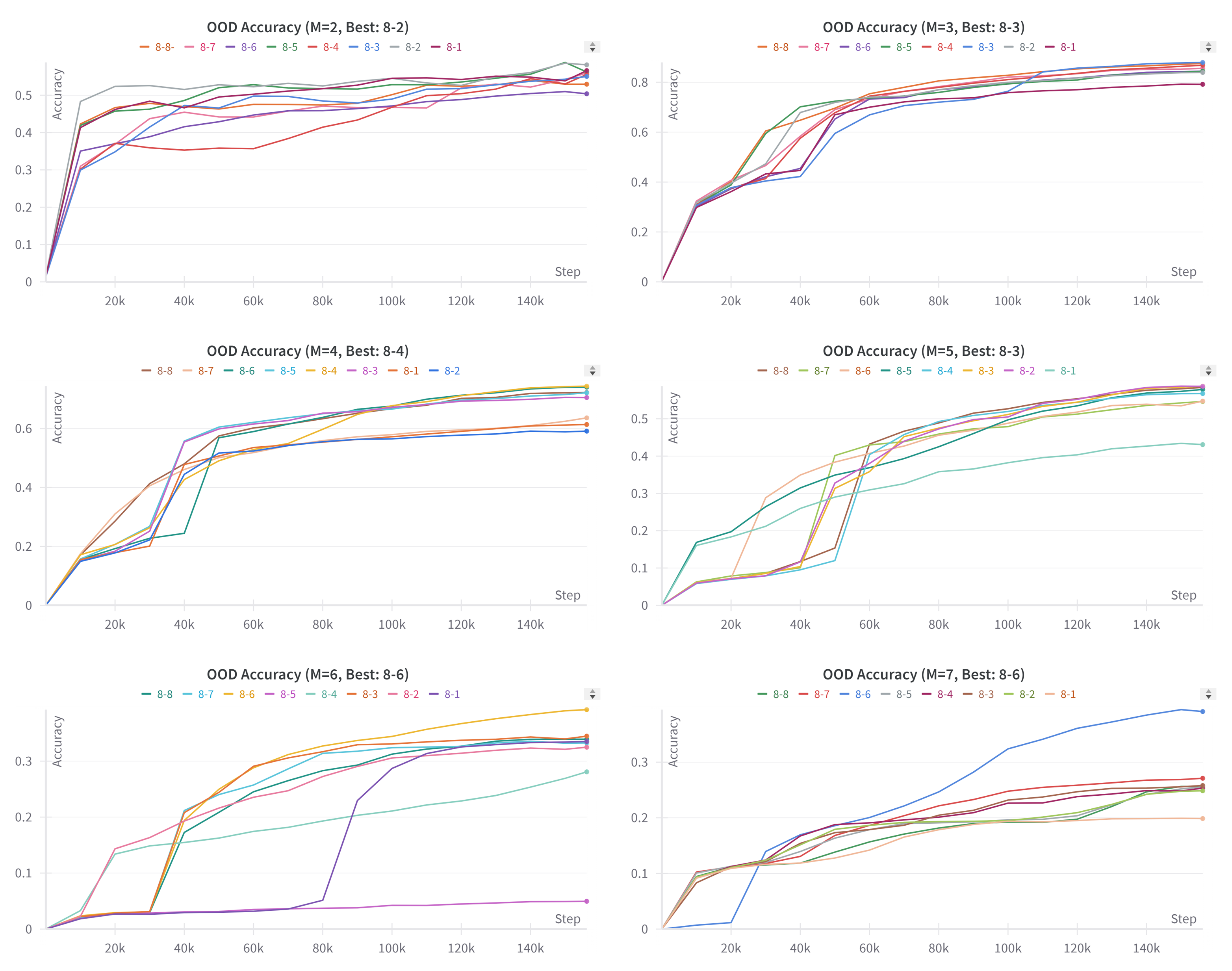}
    \caption{OOD Accuracy of training SMoE Transformer with hypernetwork linear attention.}
    \label{fig:moe_hypatt_ood}
\end{figure}
\FloatBarrier

\subsection{Evaluating SMoE-based Large Language Models on Skill-Mix}
We copied the grading metrics definitions from~\citet{yu2023skill} as a reference for the readers. Note that the first three grading metrics are very tough and most models will earn very few points when $k$ increases, as these metrics are conditioned on a specific event.

\subsubsection{Skill-Mix Grading Metrics Definition}\label{sec:supp_metrics}
Each generated text can receive up to $k + 3$ points: 1 point for each correctly illustrated skill, 1 point for sticking to the topic, 1 point for coherence / making sense, and 1 point for having at most $k - 1$ sentence. Recall that we grade each generated text three times. In each round of grading, we parse each of the criteria individually from the Grader model's output. For each criterion, we then collect the majority vote among the three grading rounds. The grading metrics are the following:
\begin{itemize}
    \item \textit{Ratio of Full Marks}: 1 if all $k + 3$ points are earned, and 0 otherwise
    \item \textit{Ratio of All Skills}: 1 if $k$ points are awarded for the $k$ skills and at least 2 points are awarded for the remaining criteria, and 0 otherwise
    \item \textit{Skill Fraction}: the fraction of points awarded for the $k$ skills if all 3 points are awarded for the remaining criteria, and 0 otherwise
    % \item \textit{Total Score:} sum of the individual points awarded
    % \item \textit{Total Skill Score:} sum of the points awarded for the $k$ skills
    % \item \textit{Rescaled Score:} $\left(\frac{c}{k+3}\right)^{k+3}$ where $c$ is the total score
\end{itemize}
We then take the maximum value of the metrics among the 3 generations for a given ($k$ skill, 1 topic) combination, and average the maximum value across all the combinations.

\newpage

\subsection{More Skill-Mix Results}\label{supp:skillmix}

\begin{table}[h!]
\vspace{-4mm}
\centering
\resizebox{0.9\textwidth}{!}{%
\begin{tabular}{|c|c|c|c|c|}
\hline
 & \multicolumn{4}{c|}{Skill-Mix results for DBRX-132B evaluated by GPT4} \\
\cline{2-5}
EPT & $k=1$ & $k=2$ & $k=3$ & $k=4$ \\
\hline
\multirow{3}{*}{1} & 0.00 $\pm$ 0.000 & 0.00 $\pm$ 0.000 & 0.00 $\pm$ 0.000 & 0.00 $\pm$ 0.000 \\
 & 0.00 $\pm$ 0.000 & 0.00 $\pm$ 0.000 & 0.00 $\pm$ 0.000 & 0.00 $\pm$ 0.000 \\
 & 0.00 $\pm$ 0.000 & 0.00 $\pm$ 0.000 & 0.00 $\pm$ 0.000 & 0.00 $\pm$ 0.000 \\
 % & 0.20 $\pm$ 0.200 & 0.40 $\pm$ 0.400 & 1.00 $\pm$ 0.316 & 0.40 $\pm$ 0.245 \\
 % & 0.00 $\pm$ 0.000 & 0.00 $\pm$ 0.000 & 0.00 $\pm$ 0.000 & 0.00 $\pm$ 0.000 \\
 % & 0.00 $\pm$ 0.001 & 0.00 $\pm$ 0.002 & 0.00 $\pm$ 0.000 & 0.00 $\pm$ 0.000 \\
\hline
\multirow{3}{*}{2} & 0.60 $\pm$ 0.245 & 0.20 $\pm$ 0.200 & 0.00 $\pm$ 0.000 & 0.00 $\pm$ 0.000 \\
 & 1.00 $\pm$ 0.000 & 0.40 $\pm$ 0.245 & 0.00 $\pm$ 0.000 & 0.00 $\pm$ 0.000 \\
 & 0.60 $\pm$ 0.245 & 0.40 $\pm$ 0.187 & 0.33 $\pm$ 0.000 & 0.15 $\pm$ 0.100 \\
 % & 3.60 $\pm$ 0.245 & 4.00 $\pm$ 0.316 & 4.00 $\pm$ 0.000 & 4.00 $\pm$ 0.447 \\
 % & 1.00 $\pm$ 0.000 & 1.40 $\pm$ 0.245 & 1.20 $\pm$ 0.200 & 1.60 $\pm$ 0.510 \\
 % & 0.73 $\pm$ 0.167 & 0.41 $\pm$ 0.155 & 0.09 $\pm$ 0.000 & 0.04 $\pm$ 0.021 \\
\hline
\multirow{3}{*}{3} & 0.80 $\pm$ 0.200 & 0.40 $\pm$ 0.245 & 0.00 $\pm$ 0.000 & 0.00 $\pm$ 0.000 \\
 & 1.00 $\pm$ 0.000 & 0.60 $\pm$ 0.245 & 0.00 $\pm$ 0.000 & 0.00 $\pm$ 0.000 \\
 & 0.80 $\pm$ 0.200 & 0.70 $\pm$ 0.122 & 0.60 $\pm$ 0.067 & 0.35 $\pm$ 0.061 \\
 % & 3.80 $\pm$ 0.200 & 4.40 $\pm$ 0.245 & 4.80 $\pm$ 0.200 & 4.40 $\pm$ 0.245 \\
 % & 1.00 $\pm$ 0.000 & 1.60 $\pm$ 0.245 & 1.80 $\pm$ 0.200 & 1.40 $\pm$ 0.245 \\
 % & 0.86 $\pm$ 0.137 & 0.60 $\pm$ 0.165 & 0.29 $\pm$ 0.049 & 0.05 $\pm$ 0.018 \\
\hline
\multirow{3}{*}{4 (default setting)} & \textbf{1.00 $\pm$ 0.000} & 0.40 $\pm$ 0.245 & 0.00 $\pm$ 0.000 & 0.00 $\pm$ 0.000 \\
 & \textbf{1.00 $\pm$ 0.000} & 0.40 $\pm$ 0.245 & 0.00 $\pm$ 0.000 & 0.00 $\pm$ 0.000 \\
 & \textbf{1.00 $\pm$ 0.000} & 0.70 $\pm$ 0.122 & 0.60 $\pm$ 0.067 & 0.40 $\pm$ 0.061 \\
 % & 4.00 $\pm$ 0.000 & 4.40 $\pm$ 0.245 & 4.80 $\pm$ 0.200 & 4.60 $\pm$ 0.245 \\
 % & 1.00 $\pm$ 0.000 & 1.40 $\pm$ 0.245 & 1.80 $\pm$ 0.200 & 1.60 $\pm$ 0.245 \\
 % & 1.00 $\pm$ 0.000 & 0.60 $\pm$ 0.165 & 0.29 $\pm$ 0.049 & 0.06 $\pm$ 0.018 \\
\hline
\multirow{3}{*}{5} & \textbf{1.00 $\pm$ 0.000} & \textbf{0.60 $\pm$ 0.245} & 0.20 $\pm$ 0.200 & 0.00 $\pm$ 0.000 \\
 & \textbf{1.00 $\pm$ 0.000} & \textbf{0.80 $\pm$ 0.200} & 0.20 $\pm$ 0.200 & 0.00 $\pm$ 0.000 \\
 & \textbf{1.00 $\pm$ 0.000} & \textbf{0.80 $\pm$ 0.122} & 0.67 $\pm$ 0.105 & 0.45 $\pm$ 0.094 \\
 % & 4.00 $\pm$ 0.000 & 4.60 $\pm$ 0.245 & 5.00 $\pm$ 0.316 & 4.80 $\pm$ 0.374 \\
 % & 1.00 $\pm$ 0.000 & 1.80 $\pm$ 0.200 & 2.00 $\pm$ 0.316 & 1.80 $\pm$ 0.374 \\
 % & 1.00 $\pm$ 0.000 & 0.73 $\pm$ 0.165 & 0.42 $\pm$ 0.153 & 0.11 $\pm$ 0.059 \\
\hline
\multirow{3}{*}{6} & 0.80 $\pm$ 0.200 & 0.40 $\pm$ 0.245 & 0.20 $\pm$ 0.200 & 0.00 $\pm$ 0.000 \\
 & 0.80 $\pm$ 0.200 & 0.40 $\pm$ 0.245 & 0.20 $\pm$ 0.200 & 0.00 $\pm$ 0.000 \\
 & 0.80 $\pm$ 0.200 & 0.70 $\pm$ 0.122 & 0.60 $\pm$ 0.163 & 0.50 $\pm$ 0.000 \\
 % & 3.60 $\pm$ 0.400 & 4.40 $\pm$ 0.245 & 5.00 $\pm$ 0.316 & 5.00 $\pm$ 0.000 \\
 % & 0.80 $\pm$ 0.200 & 1.40 $\pm$ 0.245 & 2.20 $\pm$ 0.200 & 2.00 $\pm$ 0.000 \\
 % & 0.81 $\pm$ 0.187 & 0.60 $\pm$ 0.165 & 0.42 $\pm$ 0.153 & 0.09 $\pm$ 0.000 \\
\hline
\multirow{3}{*}{7} & 1.00 $\pm$ 0.000 & 0.40 $\pm$ 0.245 & 0.20 $\pm$ 0.200 & 0.00 $\pm$ 0.000 \\
 & 1.00 $\pm$ 0.000 & 0.40 $\pm$ 0.245 & 0.20 $\pm$ 0.200 & 0.00 $\pm$ 0.000 \\
 & 1.00 $\pm$ 0.000 & 0.70 $\pm$ 0.122 & 0.67 $\pm$ 0.105 & 0.65 $\pm$ 0.061 \\
 % & 4.00 $\pm$ 0.000 & 4.40 $\pm$ 0.245 & 5.00 $\pm$ 0.316 & 5.60 $\pm$ 0.245 \\
 % & 1.00 $\pm$ 0.000 & 1.60 $\pm$ 0.245 & 2.00 $\pm$ 0.316 & 2.60 $\pm$ 0.245 \\
 % & 1.00 $\pm$ 0.000 & 0.60 $\pm$ 0.165 & 0.42 $\pm$ 0.153 & 0.24 $\pm$ 0.060 \\
\hline
\multirow{3}{*}{8} & 0.80 $\pm$ 0.200 & 0.20 $\pm$ 0.200 & 0.20 $\pm$ 0.200 & 0.00 $\pm$ 0.000 \\
 & 1.00 $\pm$ 0.000 & 0.20 $\pm$ 0.200 & 0.20 $\pm$ 0.200 & 0.00 $\pm$ 0.000 \\
 & 0.80 $\pm$ 0.200 & 0.60 $\pm$ 0.100 & 0.67 $\pm$ 0.105 & 0.50 $\pm$ 0.112 \\
 % & 3.80 $\pm$ 0.200 & 4.20 $\pm$ 0.200 & 5.00 $\pm$ 0.316 & 5.00 $\pm$ 0.447 \\
 % & 1.00 $\pm$ 0.000 & 1.20 $\pm$ 0.200 & 2.00 $\pm$ 0.316 & 2.00 $\pm$ 0.447 \\
 % & 0.86 $\pm$ 0.137 & 0.46 $\pm$ 0.134 & 0.42 $\pm$ 0.153 & 0.16 $\pm$ 0.074 \\
\hline
\multirow{3}{*}{9} & 0.80 $\pm$ 0.200 & 0.40 $\pm$ 0.245 & 0.20 $\pm$ 0.200 & 0.00 $\pm$ 0.000 \\
 & 1.00 $\pm$ 0.000 & 0.80 $\pm$ 0.200 & 0.40 $\pm$ 0.245 & 0.00 $\pm$ 0.000 \\
 & 0.80 $\pm$ 0.200 & 0.60 $\pm$ 0.187 & 0.67 $\pm$ 0.105 & 0.55 $\pm$ 0.050 \\
 % & 3.80 $\pm$ 0.200 & 4.20 $\pm$ 0.374 & 5.00 $\pm$ 0.316 & 5.20 $\pm$ 0.200 \\
 % & 1.00 $\pm$ 0.000 & 1.80 $\pm$ 0.200 & 2.20 $\pm$ 0.374 & 2.40 $\pm$ 0.245 \\
 % & 0.86 $\pm$ 0.137 & 0.55 $\pm$ 0.191 & 0.42 $\pm$ 0.153 & 0.14 $\pm$ 0.049 \\
\hline
\multirow{3}{*}{10} & 1.00 $\pm$ 0.000 & 0.20 $\pm$ 0.200 & 0.00 $\pm$ 0.000 & 0.00 $\pm$ 0.000 \\
 & 1.00 $\pm$ 0.000 & 0.60 $\pm$ 0.245 & 0.00 $\pm$ 0.000 & 0.00 $\pm$ 0.000 \\
 & 1.00 $\pm$ 0.000 & 0.60 $\pm$ 0.100 & 0.53 $\pm$ 0.082 & 0.40 $\pm$ 0.061 \\
 % & 4.00 $\pm$ 0.000 & 4.20 $\pm$ 0.200 & 4.60 $\pm$ 0.245 & 4.60 $\pm$ 0.245 \\
 % & 1.00 $\pm$ 0.000 & 1.60 $\pm$ 0.245 & 1.60 $\pm$ 0.245 & 1.60 $\pm$ 0.245 \\
 % & 1.00 $\pm$ 0.000 & 0.46 $\pm$ 0.134 & 0.24 $\pm$ 0.061 & 0.06 $\pm$ 0.018 \\
\hline
\multirow{3}{*}{11} & 1.00 $\pm$ 0.000 & 0.60 $\pm$ 0.245 & 0.20 $\pm$ 0.200 & 0.00 $\pm$ 0.000 \\
 & 1.00 $\pm$ 0.000 & 0.60 $\pm$ 0.245 & 0.20 $\pm$ 0.200 & 0.00 $\pm$ 0.000 \\
 & 1.00 $\pm$ 0.000 & 0.80 $\pm$ 0.122 & 0.47 $\pm$ 0.170 & 0.55 $\pm$ 0.094 \\
 % & 4.00 $\pm$ 0.000 & 4.60 $\pm$ 0.245 & 4.60 $\pm$ 0.400 & 5.20 $\pm$ 0.374 \\
 % & 1.00 $\pm$ 0.000 & 1.60 $\pm$ 0.245 & 2.00 $\pm$ 0.316 & 2.40 $\pm$ 0.400 \\
 % & 1.00 $\pm$ 0.000 & 0.73 $\pm$ 0.165 & 0.32 $\pm$ 0.177 & 0.18 $\pm$ 0.068 \\
\hline
\multirow{3}{*}{12} & 0.80 $\pm$ 0.200 & 0.60 $\pm$ 0.245 & 0.00 $\pm$ 0.000 & 0.00 $\pm$ 0.000 \\
 & 1.00 $\pm$ 0.000 & 0.60 $\pm$ 0.245 & 0.20 $\pm$ 0.200 & 0.00 $\pm$ 0.000 \\
 & 0.80 $\pm$ 0.200 & 0.80 $\pm$ 0.122 & 0.40 $\pm$ 0.163 & 0.50 $\pm$ 0.057 \\
 % & 3.80 $\pm$ 0.200 & 4.60 $\pm$ 0.245 & 4.40 $\pm$ 0.400 & 5.80 $\pm$ 0.490 \\
 % & 1.00 $\pm$ 0.000 & 1.60 $\pm$ 0.245 & 1.80 $\pm$ 0.490 & 2.80 $\pm$ 0.490 \\
 % & 0.86 $\pm$ 0.137 & 0.73 $\pm$ 0.165 & 0.22 $\pm$ 0.070 & 0.41 $\pm$ 0.160 \\
\hline
\multirow{3}{*}{13} & 1.00 $\pm$ 0.000 & 0.40 $\pm$ 0.245 & \textbf{0.40 $\pm$ 0.245} & 0.00 $\pm$ 0.000 \\
 & 1.00 $\pm$ 0.000 & 0.40 $\pm$ 0.245 & \textbf{0.60 $\pm$ 0.245} & 0.00 $\pm$ 0.000 \\
 & 1.00 $\pm$ 0.000 & 0.60 $\pm$ 0.187 & \textbf{0.67 $\pm$ 0.149} & 0.50 $\pm$ 0.079 \\
 % & 4.00 $\pm$ 0.000 & 4.20 $\pm$ 0.374 & 5.20 $\pm$ 0.374 & 5.00 $\pm$ 0.316 \\
 % & 1.00 $\pm$ 0.000 & 1.20 $\pm$ 0.374 & 2.40 $\pm$ 0.400 & 2.40 $\pm$ 0.510 \\
 % & 1.00 $\pm$ 0.000 & 0.55 $\pm$ 0.191 & 0.55 $\pm$ 0.189 & 0.13 $\pm$ 0.055 \\
\hline
\multirow{3}{*}{14} & 1.00 $\pm$ 0.000 & 0.60 $\pm$ 0.245 & 0.00 $\pm$ 0.000 & \textbf{0.20 $\pm$ 0.200} \\
 & 1.00 $\pm$ 0.000 & 0.60 $\pm$ 0.245 & 0.00 $\pm$ 0.000 & \textbf{0.20 $\pm$ 0.200} \\
 & 1.00 $\pm$ 0.000 & 0.60 $\pm$ 0.245 & 0.53 $\pm$ 0.082 & \textbf{0.70 $\pm$ 0.122} \\
 % & 4.00 $\pm$ 0.000 & 4.20 $\pm$ 0.490 & 4.60 $\pm$ 0.245 & 5.20 $\pm$ 0.200 \\
 % & 1.00 $\pm$ 0.000 & 1.80 $\pm$ 0.200 & 1.60 $\pm$ 0.245 & 2.40 $\pm$ 0.245 \\
 % & 1.00 $\pm$ 0.000 & 0.63 $\pm$ 0.226 & 0.24 $\pm$ 0.061 & 0.14 $\pm$ 0.049 \\
\hline
\multirow{3}{*}{15} & 1.00 $\pm$ 0.000 & 0.20 $\pm$ 0.200 & 0.00 $\pm$ 0.000 & 0.00 $\pm$ 0.000 \\
 & 1.00 $\pm$ 0.000 & 0.60 $\pm$ 0.245 & 0.00 $\pm$ 0.000 & 0.00 $\pm$ 0.000 \\
 & 1.00 $\pm$ 0.000 & 0.40 $\pm$ 0.187 & 0.53 $\pm$ 0.082 & 0.40 $\pm$ 0.061 \\
 % & 4.00 $\pm$ 0.000 & 4.20 $\pm$ 0.200 & 4.60 $\pm$ 0.245 & 4.60 $\pm$ 0.245 \\
 % & 1.00 $\pm$ 0.000 & 1.60 $\pm$ 0.245 & 1.80 $\pm$ 0.200 & 1.60 $\pm$ 0.245 \\
 % & 1.00 $\pm$ 0.000 & 0.46 $\pm$ 0.134 & 0.24 $\pm$ 0.061 & 0.06 $\pm$ 0.018 \\
\hline
\multirow{3}{*}{16} & 0.80 $\pm$ 0.200 & 0.20 $\pm$ 0.200 & 0.20 $\pm$ 0.200 & 0.00 $\pm$ 0.000 \\
 & 0.80 $\pm$ 0.200 & 0.40 $\pm$ 0.245 & 0.40 $\pm$ 0.245 & 0.00 $\pm$ 0.000 \\
 & 0.80 $\pm$ 0.200 & 0.50 $\pm$ 0.158 & 0.47 $\pm$ 0.170 & 0.25 $\pm$ 0.112 \\
 % & 3.80 $\pm$ 0.200 & 4.00 $\pm$ 0.316 & 4.60 $\pm$ 0.510 & 4.00 $\pm$ 0.447 \\
 % & 1.00 $\pm$ 0.000 & 1.40 $\pm$ 0.245 & 2.00 $\pm$ 0.447 & 1.40 $\pm$ 0.510 \\
 % & 0.86 $\pm$ 0.137 & 0.41 $\pm$ 0.155 & 0.35 $\pm$ 0.174 & 0.04 $\pm$ 0.021 \\
\hline
\end{tabular}%
}
\caption{Skill-Mix Evaluation Results on DBRX-132B~\citep{DBRX}. The grading metrics are Ratio of Full Marks/Ratio of All Skills/Skill Fraction as defined in~\Cref{sec:supp_metrics}. 'EPT' is the abbreviation for 'Number of experts per token'.}
\label{tab:DBRX_skillmix}
\end{table}

\newpage
\subsection{Power-Law Distribution of compositional task difficulties}\label{supp:powerlaw}
In this section, we describe the more complex and realistic \textit{Power-Law Distribution}~\citep{zhong2024law} setting of compositional tasks.

\subsubsection{Task-specific Variability}
We argue that the simplified Approximation Error defined in~\Cref{def:approx_error} treats every task pair uniformly, assuming \( k \) experts distribute their capacity evenly across all \( N^2 \) task combinations. This assumption, while simplifying the analysis, does not account for task-specific variability:
\begin{itemize}
    \item Some combinations might require significantly more capacity (e.g., harder-to-learn tasks).
    \item Some combinations might overlap or share features, reducing the need for dedicated experts across all \( N^2 \) pairs.
\end{itemize}
If the approximation error is task-specific, a more sophisticated construction is necessary to reflect the heterogeneity of task demands. For example, one can craft weight combinations based on difficulty or importance:
\[
O\left(\frac{\sum_{i, i'} w_{i, i'}}{k}\right),
\]
where \( w_{i, i'} \) represents the weight of importance or difficulty for task pair \( (T_i, T_{i'}) \), as recently probed empirically by \citet{zhong2024law}. Additionally, \( w_{i, i'} \)  is distributed according to a \textit{power-law}:
\[
w_{i, i'} \propto (D_i + D_{i'})^{-\alpha}, \quad \alpha > 1,
\]
where $D_i$ and $D_{i'}$ are the corresponding inverse of the difficulty or importance of compositional tasks $T_i$ and $T_{i'}$. Therefore, high weights are concentrated in a few challenging combinations (smaller \( i + i' \)), and most \( w_{i, i'} \) are near zero for large \( i + i' \). \\
Additionally, the more important/difficult task combinations will dominate, suggesting:
\begin{align*}
    \sum_{(i, i') \text{ s.t. } D_i, D_{i'} \text{are low}} w_{i, i'} \gg \sum_{(i, i') \text{ s.t. } D_i, D_{i'} \text{are high}} w_{i, i'},
\end{align*}
This suggests that \textbf{Reducing \( k \)} focuses capacity on high-weight combinations, minimizing the Approximation Error. If \( k = T \), excess experts can dilute the capacity to low-weight combinations, inflating error due to unnecessary model complexity.
\begin{remark}[Power-law distribution is more realistic in LLM Evaluation]
    The power-law distribution in task importance is more realistic and supported by empirical observations across various studies. \citet{dziri2024faith} highlights how transformers prioritize high-frequency patterns, leaving rare patterns underrepresented, reflecting power-law dynamics. Similarly, \citet{yu2023skill} demonstrates the combinatorial explosion of rare skill requirements, where a few dominant combinations account for most of the performance. Lastly, \citet{chen2024skill} emphasizes skill hierarchies, where foundational skills dominate model capacity, mirroring power-law behavior in skill distributions. Together, these findings hint that activating all experts is practically inefficient for sparse or compositional tasks, as it dilutes capacity across less-relevant combinations.
\end{remark}

% Under these settings, it is reasonable to defined the Approximation Error as following:
% \begin{definition}[Approximation Error Under Power-Law Distribution] Suppose Assumption~\ref{ass:approx1} and Power-Law distribution holds
%     \[
%     E_{\text{Approximation}}(k) = \sum_{i, i'} \frac{w_{i, i'}}{k},
%     \]
% \end{definition}

\newpage
\section{Complete Proof of main result}
\label{sec:app:proofs}

We present the proof of our main result as below:
\begin{proof}
Following the classical PAC learning theory, the main objective is to show the following probabilistic bound of $\sup_{f \in \Set{F}} \left\lvert \LS(f) - \LD(f) \right\rvert$.
First we use ghost sampling trick: we draw an i.i.d. copies of training samples: $\Set{S'} = \{\Mat{x'}_1, \cdots, \Mat{x'}_m\} \overset{i.i.d.}{\sim} \Set{D}^m$.
Then by Lemma \ref{lem:ghost_sample}, and setting $e^{-\frac{1}{2} \epsilon^2 m} \le 1/4$, we have
\begin{align}
\Prob \left(\sup_{f \in \Set{F}} \left\lvert \LS(f) - \LD(f) \right\rvert \ge \epsilon \right) &\le 2 \Prob \left(\sup_{f \in \Set{F}} \left\lvert \LS(f) - \LSp(f) \right\rvert \ge \frac{\epsilon}{2} \right)
\end{align}
Then our proof proceeds by reformulating the gating function $a(\Mat{x})$.
Let us rewrite $a(\Mat{x}) = \mu(\Mat{x}) \odot \nu(\Mat{x})$, where $\odot$ denotes the element-wise multiplication, $\mu(\Mat{x}): \real^{d} \rightarrow \{0, 1\}^T$ produces a binary mask specifying the sparse expert selection ($\lVert \mu(\Mat{x}) \rVert_0 = k$), and $\nu(\Mat{x}): \real^{d} \rightarrow \real_+^T$ outputs the normalized weights for selected experts such that $\lVert a(\Mat{x}) \rVert_1 = 1$. In particular, we note that $\nu(\Mat{x})$ is dependent of the function $\mu(\Mat{x})$. We define $\Set{V}\lvert_{\mu}$ as the class of $g$ induced by $\Set{A}$ and $\mu(\Mat{x})$.

Now notice that $\mu(\Mat{x})$ amounts to a multi-class classifier, which maps the input $\Mat{x}$ to one of the sparse patterns $\Set{M}(2m)$.
Define $\mu_1, \cdots, \mu_{\Gamma}$ which shatters all the possible sparse patterns produced by $2m$ data samples, then we have 
\begin{align}
\Prob \left(\sup_{f \in \Set{F}} \left\lvert \LS(f) - \LSp(f) \right\rvert \ge \frac{\epsilon}{2} \right) &= \Prob \left(\sup_{a \in \Set{A}} \sup_{\substack{h_j \in \Set{H}, \\ \forall j \in [T]}} \left\lvert \LS(f) - \LSp(f) \right\rvert \ge \frac{\epsilon}{2} \right) \\
&\le \Prob \left(\sup_{\mu \in \{\mu_1, \cdots, \mu_{\Gamma}\}} \sup_{\substack{h_j \in \Set{H}, \forall j \in [T] \\ \nu \in \Set{V}\lvert_{\mu}}} \left\lvert \LS(f) - \LSp(f) \right\rvert \ge \frac{\epsilon}{2} \right) \label{eqn:discrete_mask_space} \\
&\le \sum_{t = 1}^{\Gamma} \Prob \left( \left. \sup_{\substack{h_j \in \Set{H}, \forall j \in [T] \\ \nu \in \Set{V}\lvert_{\mu}}} \left\lvert \LS(f) - \LSp(f) \right\rvert \ge \frac{\epsilon}{2} \right| \mu = \mu_t \right) \\
&\le \Gamma \sup_{\mu^* \in \{\mu_1, \cdots, \mu_{\Gamma}\}} \Prob \left( \left. \sup_{\substack{h_j \in \Set{H}, \forall j \in [T] \\ v \in \Set{V}\lvert_{\mu}}} \left\lvert \LS(f) - \LSp(f) \right\rvert \ge \frac{\epsilon}{2} \right| \mu = \mu^* \right) \\
&\le 2 \Gamma \sup_{\mu^* \in \{\mu_1, \cdots, \mu_{\Gamma}\}} \Prob \left( \left. \sup_{\substack{h_j \in \Set{H}, \forall j \in [T] \\ v \in \Set{V}\lvert_{\mu}}} \LS(f) - \LSp(f) \ge \frac{\epsilon}{2} \right| \mu = \mu^* \right) \label{eqn:remove_abs},
\end{align}
where we apply union bound to obtain Eq. \ref{eqn:discrete_mask_space} and Eq. \ref{eqn:remove_abs}.
Next, we do counting to bound $\Gamma$. Since $\mu$ is essentially a multi-class classifier with $T \choose k$ classes, by Lemma \ref{lem:nataranjan}, we plug in the Natarajan dimension of our sparse patterns:
\begin{align} \label{eqn:growth_bound}
\Gamma \le {T \choose k}^{2Nd_N} \cdot (2m)^{Nd_N}.
\end{align}
On the other hand, we bound remaining probability term by examining the expectation given a fixed masking function $\mu$. Define function $\phi\lvert_{\mu}$ conditioned on $\mu$ as:
\begin{align} \label{eqn:phi_def}
\phi\lvert_{\mu}(\Set{S}, \Set{S'}) = \sup_{\substack{h_j \in \Set{H}, \forall j \in [T] \\ \nu \in \Set{V}\lvert_{\mu}}} \LS(f) - \LSp(f)    
\end{align}
By Lemma \ref{lem:self_bounding} and Mcdiarmid's inequality, for any $\mu$, we have bound:
\begin{align}
\Prob \left( \phi\lvert_{\mu}(\Set{S}, \Set{S'}) \ge \frac{\epsilon}{2} \right) &= \Prob \left( \phi\lvert_{\mu}(\Set{S}, \Set{S'}) - \mean\left[ \phi\lvert_{\mu}(\Set{S}, \Set{S'}) \right] \ge \frac{\epsilon}{2} - \mean\left[ \phi\lvert_{\mu}(\Set{S}, \Set{S'}) \right] \right) \\
&\le \exp\left( -2m\left(\frac{\epsilon}{2} - \mean\left[ \phi\lvert_{\mu}(\Set{S}, \Set{S'}) \right]\right)^2 \right) \label{eqn:mcdiarmid_bound}
\end{align}
Combined with Eq. \ref{eqn:remove_abs}, \ref{eqn:growth_bound}, \ref{eqn:mcdiarmid_bound}, we state, with probability at least $1 - \delta$,
\begin{align}
\phi\lvert_{\mu}(\Set{S}, \Set{S'}) &\le 2 \mean\left[ \phi\lvert_{\mu}(\Set{S}, \Set{S'}) \right] + 2\sqrt{\frac{\log\left({T \choose k}^{2 N d_N} (2m)^{N d_N} \right) + \log(2 / \delta)}{2m}} \\
&= 2\mean\left[ \phi\lvert_{\mu}(\Set{S}, \Set{S'}) \right] + 2\sqrt{\frac{2 k N d_N \log T + Nd_N \log(2m) + \log(2 / \delta)}{2m}}
\end{align}
We conclude the proof by the bounding $\mean\left[ \phi\lvert_{\mu}(\Set{S}, \Set{S'}) \right] \le 2C \radem_m(\Set{H})$ using Lemma \ref{lem:bound_mean}.
\end{proof}

\begin{lemma} \label{lem:bound_mean}
Consider C-Lipschitz loss function: $\ell: \Set{Y} \times \real \rightarrow \real$, and follow the definition of $\phi\lvert_{\mu}$ in Eq. \ref{eqn:phi_def}, we have
\begin{align}
\mean\left[ \phi\lvert_{\mu}(\Set{S}, \Set{S'}) \right] \le 2C \radem_m(\Set{H})
\end{align}
\end{lemma}
\begin{proof}
For the sake of notation simplicity, we define a function space conditioned on a masking function $\mu$:
\begin{align}
\Set{F}\lvert_{\mu} = \left\{ f(\Mat{x}) = \sum_{j=1}^{T} \mu(\Mat{x})_j \nu(\Mat{x})_j h_j(\Mat{x}) : h_1, \cdots, h_T \in \Set{H}, \nu \in \Set{V}\lvert_{\mu} \right\}
\end{align}
We denote the loss function $\ell$ composed on $\Set{F}\lvert_{\mu}$ as $\ell \circ \Set{F}\lvert_{\mu} = \left\{ \ell(f(\Mat{x})) : f \in \Set{F}\lvert_{\mu} \right\}$.
By Lemma \ref{lem:radem_bound},
\begin{align}
\mean \left[ \phi\lvert_{u}(\Set{S}, \Set{S'}) \right] &= \mean \left[ \sup_{\ell_f \in \ell \circ \Set{F}\lvert_{\mu}} \left( \frac{1}{m} \sum_{i=1}^{m} \ell_f(\Mat{x}_i) - \frac{1}{m} \sum_{i=1}^{m} \ell_f(\Mat{x'}_i) \right) \right] \\
&\le 2\radem_m(\ell \circ \Set{F}\lvert_{\mu}) \label{eqn:E_radem_bound}
\end{align}
Since $\ell$ is Lipschitz function, by Lemma \ref{lem:lips_radem}, we have
\begin{align} \label{eqn:lips_radem}
\radem_m(\ell \circ \Set{F}\lvert_{\mu}) \le C \radem_m(\Set{F}\lvert_{\mu}) 
\end{align}
Afterwards, we bound $\radem_m(\Set{F}\lvert_{\mu})$ by:
\begin{align}
\mean_{\Set{S}, \Mat{\sigma}} \left[ \frac{1}{m} \sup_{f \in \Set{F}\lvert_{\mu}} \sum_{i=1}^{m} \sigma_i f(\Mat{x}_i) \right] &= \mean_{\Set{S}, \Mat{\sigma}} \left[ \sup_{\sup_{\substack{h_j \in \Set{H}, \forall j \in [T] \\ v \in \Set{V}\lvert_{\mu}}}} \frac{1}{m} \sum_{i=1}^{m} \sigma_i \sum_{j=1}^{T} \mu(\Mat{x}_i)_j \nu(\Mat{x}_i)_j h_j(\Mat{x}_i) \right] \\
&\le \mean_{\Set{S}, \Mat{\sigma}} \left[ \sup_{\substack{h_j \in \Set{H}, \forall j \in [T] \\ \Mat{\lambda} \in \real_+^T, \lVert \Mat{\lambda} \rVert_1 = 1}} \frac{1}{m} \sum_{i=1}^{m} \sigma_i \sum_{j=1}^{T} \Mat{\lambda}_j h_j(\Mat{x}_i) \right] \label{eqn:relax_simplex}\\
&= \radem_m(\Set{H}) \label{eqn:conv_radem},
\end{align}
where we notice that $\sum_{j=1}^{T} \mu(\Mat{x})_j \nu(\Mat{x}_j) = 1$ due to the softmax normalization over the weights of selected experts, then Eq. \ref{eqn:relax_simplex} can be relaxed by supremum over all simplex. The last equation follows from Lemma \ref{lem:convex_hull_radem}.
Now we can conclude the proof by combining Eq. \ref{eqn:E_radem_bound}, \ref{eqn:lips_radem}, and \ref{eqn:conv_radem}.
\end{proof}

\begin{lemma}[Natarajan Lemma \citep{natarajan1989learning}] \label{lem:nataranjan}
Given a set of finite data points $\Set{S}$ with $\lvert \Set{S} \rvert = m$, and a hypothesis space $\Set{H}$ of functions $\Set{S} \rightarrow [k]$ with Natarajan dimension $d_N$, then the growth function is bounded by:
\begin{align}
\tau_{\Set{H}}(m) \le m^{d_N} \cdot k^{2d_N}
\end{align}
\end{lemma}
\begin{proof}
See \cite{natarajan1989learning}.
\end{proof}

%Lemma from Lec10, Page 3%
\begin{lemma}[Ghost Sampling \citep{shalev-shwartz_ben-david_2022}] \label{lem:ghost_sample}
Given $\Set{S}$ and $\Set{S^\prime}$ with $\left | \Set{S} \right |= \left | \Set{S^\prime} \right | = m$, we have the following inequality for any hypothesis space $\Set{H}$:
\begin{align}
\left(1-2 e^{-\frac{1}{2} \epsilon^2 m}\right) \Prob\left[\sup _{h \in \Set{H}}\left|\LS(h)-\LSp(h)\right|>\epsilon\right] \leq \Prob\left[\sup _{h \in \mathcal{H}}\left|\LS(h)-\LSp(h)\right|>\frac{\epsilon}{2}\right]
\end{align}
\end{lemma}

\begin{proof}
We note that $\Prob \left[ \sup _{h \in \Set{H}} \left| \LS(h) - \LD(h)  \right| > \epsilon \right] > 0$, then

\begin{align}
    &\Prob \left[ \sup _{h \in \Set{H}} \left| \LS(h) - \LSp(h)  \right| > \frac{\epsilon}{2} \right] \\
    &\geq  \Prob \left[ \sup _{h \in \Set{H}} \left| \LS(h) - \LSp(h)  \right| > \frac{\epsilon}{2} \cap \sup _{h \in \Set{H}} \left| \LS(h) - \LD(h)  \right| > \epsilon \right]\\
    &= \Prob \left[ \sup _{h \in \Set{H}} \left| \LS(h) - \LSp(h)  \right| > \frac{\epsilon}{2} \right] \times \Prob \left[ \sup _{h \in \Set{H}} \left| \LS(h) - \LSp(h)  \right| > \frac{\epsilon}{2} \middle|  \sup _{h \in \Set{H}} \left| \LS(h) - \LD(h)  \right| > \epsilon \right]
\end{align}
Fix the dataset $\Set{S}$ for the event on which we are conditioning. Let $h^*$ be any hypothesis for which $\left| \LS(h) - \LD(h) \right| > \epsilon$, then:
\begin{align}
    &\Prob \left[ \sup _{h \in \Set{H}} \left| \LS(h) - \LSp(h)  \right| > \frac{\epsilon}{2} \middle|  \sup _{h \in \Set{H}} \left| \LS(h) - \LD(h)  \right| > \epsilon \right]\\
    &\geq \Prob \left[ \sup _{h \in \Set{H}} \left| \LS(h^*) - \LSp(h^*)  \right| > \frac{\epsilon}{2} \middle|  \sup _{h \in \Set{H}} \left| \LS(h) - \LD(h)  \right| > \epsilon \right]\\
    &\geq \Prob \left[ \sup _{h \in \Set{H}} \left| \LSp(h^*) - \LD(h^*)  \right| \leq \frac{\epsilon}{2} \middle|  \sup _{h \in \Set{H}} \left| \LS(h) - \LD(h)  \right| > \epsilon \right] \\
    &\geq 1-2 e^{-\frac{1}{2} \epsilon^2 m},
\end{align}
where the last inequality follows from Hoeffding's inequality. Following an averaging over $\Set{S}$ argument, we conclude that:
\begin{align}
\left(1-2 e^{-\frac{1}{2} \epsilon^2 m}\right) \Prob\left[\sup _{h \in \Set{H}}\left|\LS(h)-\LSp(h)\right|>\epsilon\right] \leq \Prob\left[\sup _{h \in \mathcal{H}}\left|\LS(h)-\LSp(h)\right|>\frac{\epsilon}{2}\right]
\end{align}
\end{proof}

\begin{lemma} \label{lem:radem_bound}
Given any funtion class $\Set{F}$, for any $\Set{S}$ and $\Set{S'}$ drawn i.i.d. from $\Set{D}^m$ with $|\Set{S}| = |\Set{S'}| = m$, it holds that
\begin{align}
\mean_{\Set{S}, \Set{S'}} \left[ \sup_{f \in \Set{F}} \left( \frac{1}{m} \sum_{i=1}^{m} f(\Mat{x}_i) - \frac{1}{m} \sum_{i=1}^{m} f(\Mat{x'}_i) \right) \right] \le 2 \radem_m(\Set{F})
\end{align}
\end{lemma}

\begin{proof}
The proof is concluded by the following derivation:
\begin{align}
    \mean \left[ \sup_{f \in \Set{F}} \LS(f) - \LSp(f) \right]
    &= \mean_{\Set{S}} \left[ \mean_{\Set{S^\prime}} \left[\sup_{f \in \Set{F}} \left\lvert \LS(f) - \LSp(f) \right\rvert  \right]   \right]\\
    &\leq \mean_{\Set{S},\Set{S^\prime}} \left[ \mean_{\sigma_{i}} \left[ \sup_{f \in \Set{F}} \left(\frac{1}{m} \sum_{i=1}^m \left(f(z_i) - \frac{1}{m}\sum_{i=1}^m f(z_i^\prime) \right) \right) \right]  \right] \label{eqn:intro_radem_vars} \\
    &\leq \mean_{\Set{S}, \Set{S^\prime}, \sigma_i} \left[ \sup _{f \in \Set{F}}\left(\frac{1}{m} \sum_{i=1}^m \sigma_i f\left(z_i\right)\right)+\sup _{f \in \Set{F}} \left(\frac{1}{m} \sum_{i=1}^m-\sigma_i f\left(z_i^{\prime}\right)\right) 
    \right]\\
    &= 2 \radem_m(\Set{F}),
\end{align}
where we introduce Radamacher random variables $\sigma_i, i = 1, \cdots, m$ in Eq. \ref{eqn:intro_radem_vars}.
\end{proof}

\begin{lemma} \label{lem:lips_radem}
Suppose $\Set{H} \subseteq \{h: \Set{X} \rightarrow \Set{Y} \}$ and function $\ell: \Set{Y} \times \real \rightarrow \real$ is a C-Lipschitz function, define define $\ell \circ \Set{H} = \{\ell \circ h: \forall h \in \Set{H}\}$, then $\radem_m(\ell \circ \Set{H}) \le C\radem_m(\Set{H})$.
\end{lemma}

\begin{proof}
    See \citet{meir_zhang_2003}.
\end{proof}

\begin{lemma} \label{lem:convex_hull_radem}
Suppose $\Set{H} \subseteq \{h: \Set{X} \rightarrow \Set{Y} \}$, then all functions constructed by convex combinations of $\Set{H}$ satisfies:
\begin{align}
\mean_{\Set{S}, \Mat{\sigma}} \left[ \sup_{\substack{h_j \in \Set{H}, \forall j \in [T] \\ \Mat{\lambda} \in \real_+^T, \lVert \Mat{\lambda} \rVert_1 = 1}} \frac{1}{m} \sum_{i=1}^{m} \sigma_i \sum_{j=1}^{T} \Mat{\lambda}_j h_j(\Mat{x}_i) \right] = \radem_m(\Set{H}).
\end{align}
\end{lemma}
\begin{proof}
\begin{align}
\mean_{\Set{S}, \Mat{\sigma}} \left[ \sup_{\substack{h_j \in \Set{H}, \forall j \in [T] \\ \Mat{\lambda} \in \real_+^T, \lVert \Mat{\lambda} \rVert_1 = 1}} \frac{1}{m} \sum_{i=1}^{m} \sigma_i \sum_{j=1}^{T} \Mat{\lambda}_j h_j(\Mat{x}_i) \right] &= \mean_{\Set{S}, \Mat{\sigma}} \left[ \sup_{\substack{h_j \in \Set{H}, \\ \forall j \in [T]}} \sup_{\substack{\Mat{\lambda} \in \real_+^T, \\ \lVert \Mat{\lambda} \rVert_1 = 1}} \frac{1}{m} \sum_{j=1}^{T} \Mat{\lambda}_j \left(\sum_{i=1}^{m} \sigma_i h_j(\Mat{x}_i)\right) \right] \\
&= \mean_{\Set{S}, \Mat{\sigma}} \left[ \sup_{h_{j^*} \in \Set{H}} \frac{1}{m} \sum_{i=1}^{m} \sigma_i h_{j^*}(\Mat{x}_i) \right] \label{eqn:conv_selec} \\
&= \radem_m(\Set{H}),
\end{align}
where Eq. \ref{eqn:conv_selec} uses the fact that $\sum_{j=1}^{T} \Mat{\lambda}_j y_j \le \max_{j=1,\cdots,T} y_j$ for any convex coefficients $\Mat{\lambda}$. Moreover, the equality is achieved if and only if $\Mat{\lambda}_{j} = 1$ for $j = \argmax_{j=1,\cdots,T} y_j$ and $\Mat{\lambda}_{j} = 0$ otherwise.
\end{proof}

\begin{lemma} \label{lem:self_bounding}
For arbitrary loss function $\ell: \Set{Y} \times \real \rightarrow [0, 1]$, hypothesis space $\Set{H}$, and any $\Set{S}, \Set{S'} \subseteq \Set{X}$  $|\Set{S}| = |\Set{S'}| = m$, the function
\begin{align}
\phi(\Set{S}, \Set{S'}) = \sup_{h \in \Set{H}} \LS(f) - \LSp(f)
\end{align}
satisfies that
\begin{align}
\lvert \phi(\Set{S}, \Set{S'}) - \phi(\widehat{\Set{S}}, \widehat{\Set{S'}}) \rvert \le \frac{1}{m},
\end{align}
where $\widehat{\Set{S}}$ and $\widehat{\Set{S'}}$ changes one element either in $\Set{S}$ or $\Set{S'}$ but not both, i.e., if $\Set{S}$ is changed to $\widehat{\Set{S}}$, then $\Set{S'}$ remains unchanged, and vice versa.
\end{lemma}
    
\begin{proof}
    Since $f\left(\cdot\right) \in[0,1]$, then it is obvious that $\lvert \phi(\Set{S}, \Set{S'}) - \phi(\widehat{\Set{S}}, \widehat{\Set{S'}}) \rvert \le \frac{1}{m}$ since we can flip at most one element.
\end{proof}

\end{document}